\documentclass[10pt]{article}

\PassOptionsToPackage{table}{xcolor}

\usepackage{scai}

\definecolor{promptblue}{HTML}{2B5B9E}
\definecolor{prompttea}{HTML}{1F7A6F}
\definecolor{promptplum}{HTML}{8A5CD6}
\definecolor{promptamber}{HTML}{B45A1C}
\definecolor{promptolive}{HTML}{556B2F}
\definecolor{citepink}{HTML}{E63E91}

\hypersetup{
  colorlinks=true,
  citecolor=promptblue,
  linkcolor=promptplum,
  urlcolor=citepink,
}

\usepackage[authoryear, sort&compress, round]{natbib}

\usepackage{booktabs}
\usepackage{multirow}
\usepackage{tabularx}
\newcolumntype{C}{>{\centering\arraybackslash}X}
\usepackage{amsthm}
\newtheorem{proposition}{Proposition}
\usepackage{nicefrac}
\usepackage{placeins}
\usepackage{wrapfig}
\usepackage{listings}
\usepackage{float}

\makeatletter
\g@addto@macro\normalsize{%
  \setlength{\abovedisplayskip}{3pt plus 1pt minus 1pt}%
  \setlength{\abovedisplayshortskip}{0pt plus 1pt}%
  \setlength{\belowdisplayskip}{3pt plus 1pt minus 1pt}%
  \setlength{\belowdisplayshortskip}{2pt plus 1pt minus 1pt}%
}
\g@addto@macro\small{%
  \setlength{\abovedisplayskip}{2pt plus 1pt minus 1pt}%
  \setlength{\abovedisplayshortskip}{0pt plus 1pt}%
  \setlength{\belowdisplayskip}{2pt plus 1pt minus 1pt}%
  \setlength{\belowdisplayshortskip}{1pt plus 1pt minus 1pt}%
}
\makeatother
\normalsize

\tcbuselibrary{listings}

\titlecontents{section}
  [1.6em]
  {\addvspace{6pt}\bfseries}
  {\contentslabel[\thecontentslabel]{1.6em}}
  {\hspace*{-1.6em}}
  {\hfill\contentspage}
\titlecontents{subsection}
  [4em]
  {\addvspace{2pt}\normalfont\small}
  {\contentslabel[\thecontentslabel]{2.4em}}
  {\hspace*{-2.4em}}
  {\titlerule*[6pt]{.}\contentspage}

\newtcolorbox{promptbox}[2][promptblue]{
  enhanced,
  colback=#1!4!white,
  colframe=#1!80!black,
  colbacktitle=#1!85!black,
  coltitle=white,
  fonttitle=\bfseries\small,
  title={#2},
  boxrule=0.6pt,
  titlerule=0pt,
  arc=3pt,
  left=8pt, right=8pt, top=6pt, bottom=6pt,
  fontupper=\small,
}

\lstdefinestyle{jsontool}{
  basicstyle=\scriptsize\ttfamily,
  breaklines=true,
  breakatwhitespace=true,
  showstringspaces=false,
  columns=fullflexible,
  keepspaces=true,
  stringstyle=\color{promptplum!85!black},
  keywordstyle=\color{promptblue!85!black}\bfseries,
  commentstyle=\color{gray!70!black}\itshape,
  morestring=[b]",
  morecomment=[l]{//},
  morekeywords={true,false,null},
  literate={:}{{\textcolor{black}{:}}}1,
}

\newtcblisting{toolsjson}[2][promptolive]{
  enhanced,
  colback=#1!4!white,
  colframe=#1!80!black,
  colbacktitle=#1!85!black,
  coltitle=white,
  fonttitle=\bfseries\small,
  title={#2},
  boxrule=0.6pt,
  titlerule=0pt,
  arc=3pt,
  left=8pt, right=8pt, top=6pt, bottom=6pt,
  listing only,
  listing options={style=jsontool},
}

\renewcommand{\authorfont}{\normalsize\rmfamily}

\makeatletter
\renewcommand{\scaiauthor}[2]{%
  \stepcounter{scai@numauthors}%
  \ifnum\value{scai@numauthors}=1
    \gdef\@scai@authorlist{\mbox{\textbf{#2}\textsuperscript{\@scai@rendertags{#1}}}}%
  \else
    \g@addto@macro\@scai@authorlist{, \mbox{\textbf{#2}\textsuperscript{\@scai@rendertags{#1}}}}%
  \fi
}

\renewcommand{\@scai@renderaffils}{%
  \begingroup
  \normalsize
  \def\@scai@affilfirst{1}%
  \@for\@scai@tmp:=\@scai@affillist\do{%
    \ifnum\@scai@tmp=3\relax \par\def\@scai@affilfirst{1}\fi
    \if\@scai@affilfirst 1\def\@scai@affilfirst{0}\else ,\space\fi
    \textsuperscript{\@scai@tmp}%
    \csname @scai@affil@\@scai@tmp\endcsname
  }%
  \par
  \endgroup
}
\makeatother

\makeatletter
\newcommand{\authorbreak}{%
  \g@addto@macro\@scai@authorlist{\\}%
  \let\@scai@authorsaved\scaiauthor
  \def\scaiauthor##1##2{%
    \stepcounter{scai@numauthors}%
    \g@addto@macro\@scai@authorlist{\mbox{\textbf{##2}\textsuperscript{\@scai@rendertags{##1}}}}%
    \let\scaiauthor\@scai@authorsaved
  }%
}
\makeatother

\makeatletter
\renewcommand{\maketitle}{%
  \bgroup
  \setlength{\parindent}{0pt}%
  \setlength{\parskip}{0pt}%
  \vspace*{3pt}%
  {\centering\titlefont\@title\par}%
  \vspace{14pt}%
  {\centering\authorfont\@scai@authorlist\par}%
  \vspace{8pt}%
  {\centering\@scai@renderaffils\par}%
  \if@scai@haskeywords
    \vspace{6pt}%
    {\centering\keywordsfont{\sffamily\textbf{Keywords:}} \@scai@keywords\par}%
  \fi
  \vspace{10pt}%
  \@scai@renderabstract
  \egroup
  \thispagestyle{scaifirstpage}%
}
\makeatother

\title{OpenSearch-VL: An Open Recipe for Frontier\\Multimodal Search Agents}

\contribnote{corr}{$\dagger$}{Corresponding author.}

\scaiauthor{1,2}{Shawn Chen}
\scaiauthor{3}{Kaituo Feng}
\scaiauthor{1}{Hangting Chen}
\scaiauthor{3}{Wenxuan Huang}
\scaiauthor{3}{Dasen Dai}
\scaiauthor{2,4}{Quanxin Shou}
\authorbreak
\scaiauthor{3}{Yunlong Lin}
\scaiauthor{3}{Xiangyu Yue}
\scaiauthor{4}{Shenghua Gao}
\scaiauthor{1,corr}{Tianyu Pang}

\affiliation{1}{Tencent Hunyuan}
\affiliation{2}{University of California, Los Angeles}
\affiliation{3}{The Chinese University of Hong Kong}
\affiliation{4}{The University of Hong Kong}

\metadata[\raisebox{-0.35ex}{\includegraphics[height=1.3em]{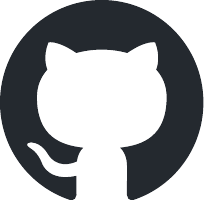}}\,Home]{\url{https://github.com/shawn0728/OpenSearch-VL}}
\metadata[\raisebox{-0.35ex}{\includegraphics[height=1.3em]{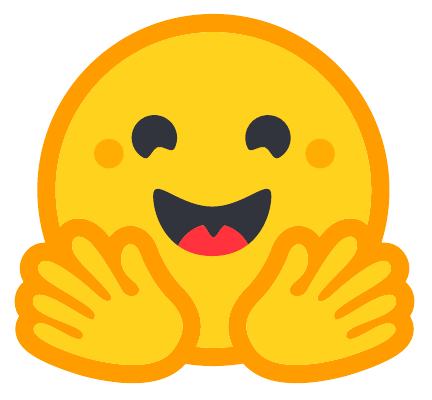}}\,Huggingface]{\url{https://huggingface.co/OpenSearch-VL}}

\begin{abstract}
Deep search has become a crucial capability for frontier multimodal agents, enabling models to solve complex questions through active search, evidence verification, and multi-step reasoning. Despite rapid progress, top-tier multimodal search agents remain difficult to reproduce, largely due to the absence of open high-quality training data, transparent trajectory synthesis pipelines, or detailed training recipes. To this end, we introduce \textbf{OpenSearch-VL}, a fully open-source recipe for training frontier multimodal deep search agents with agentic reinforcement learning. First, we curated a dedicated pipeline to construct high-quality training data through Wikipedia path sampling, fuzzy entity rewriting, and source-anchor visual grounding, which jointly reduce shortcuts and one-step retrieval collapse. Based on this pipeline, we curate two training datasets, \emph{SearchVL-SFT-36k} for SFT and \emph{SearchVL-RL-8k} for RL. Besides, we design a diverse tool environment that unifies text search, image search, OCR, cropping, sharpening, super-resolution, and perspective correction, enabling agents to combine active perception with external knowledge acquisition. Finally, we propose a multi-turn fatal-aware GRPO training algorithm that handles cascading tool failures by masking post-failure tokens while preserving useful pre-failure reasoning through one-sided advantage clamping. Built on this recipe, OpenSearch-VL delivers substantial performance gains, with over 10-point average improvements across seven benchmarks, and achieves results comparable to proprietary commercial models on several tasks. We will release all data, code, and models to support open research on multimodal deep search agents.
\end{abstract}

\begin{document}

\maketitle

\section{Introduction}
\label{sec:introduction}

Multimodal deep search has emerged as a critical direction for multimodal large language models (MLLMs), enabling them to evolve from passive visual understanding systems into agents that actively search evidence, verify facts, and reason over knowledge-intensive visual queries \citep{huang2026vision,feng2026gen,chen2026unify}.
However, frontier multimodal search agents remain difficult to reproduce, as their training data, code are often proprietary or insufficiently disclosed \citep{seed2.0,huang2026vision,singh2025openai, kimiteam2026kimik25visualagentic}. As a result, the community still lacks a fully open recipe for building, analyzing, and improving strong multimodal search agents.

Among these missing components, high-quality training data is a central bottleneck. The strongest frontier systems are still largely dominated by well-funded commercial corporations \citep{Claude_4,comanici2025gemini}, where the data sources, filtering criteria, expert demonstrations, and tool-use trajectories are typically kept private. 
This makes it difficult to reproduce advanced multimodal search capabilities or systematically study which data properties are essential for agentic search behavior.
The issue is even more pronounced in multimodal settings, where effective training data must capture image-grounded understanding, multi-hop retrieval, evidence verification, and long-horizon tool use rather than simple visual question answering. 
Therefore, releasing high-quality training data is crucial for making frontier multimodal search agent research more transparent, reproducible, and accessible.

Beyond data, training multimodal search agents also poses unique challenges, especially when applying agentic reinforcement learning (agentic RL) \citep{fan2026exploring,webwatcher,huang2026vision} to long-horizon tool-use settings. 
Agentic search trajectories involve multiple rounds of reasoning, tool invocation, and observation integration, where a single malformed call, timeout, irrelevant query, or repeated failure can invalidate the remaining rollout. 
Simply discarding such trajectories wastes useful pre-failure reasoning, while training on the full rollout introduces noisy gradients from meaningless post-failure tokens. 
Another practical challenge is that real-world visual inputs are often imperfect, such as blurred photos, low-resolution thumbnails, skewed documents, and crowded screenshots. 
In these cases, searching alone is insufficient, and the agent must first crop, enhance, rectify, or parse the visual evidence before reliable search can begin. 
However, most existing multimodal search agents focus mainly on retrieval and do not jointly address robust visual pre-processing and failure-aware long-horizon RL.

In this work, we introduce \textbf{OpenSearch-VL}, a fully open recipe for training frontier multimodal deep search agents with agentic RL. 
Our recipe addresses the above challenges from data, tools, and training. 
First, we develop a dedicated data curation pipeline to build high-quality training data. 
Starting from the Wikipedia hyperlink graph, we sample multi-hop entity paths and convert them into multi-hop VQA instances by rewriting intermediate entities into fuzzy descriptions, followed by a carefully designed filtering mechanism.
This design avoids single-hop image lookup shortcuts and encourages the agent to learn multi-hop search and reasoning behaviors.
This pipeline yields two training datasets \textbf{SearchVL-SFT-36k} for SFT and \textbf{SearchVL-RL-8k} for agentic RL.
Second, we build a tool environment that goes beyond retrieval-only multimodal agent. 
In addition to search, the agent is equipped with OCR, cropping, sharpening, super-resolution, and perspective correction, allowing it to handle imperfect visual inputs in real-world scenarios before querying external knowledge. 
Finally, we develop an agentic RL algorithm based on GRPO ~\citep{guo2025deepseek} for long-horizon multimodal tool use, where multi-step interactions often lead to cascading tool failures. To address this issue, we introduce fatal-aware token masking that removes invalid post-failure suffixes from optimization, while preserving useful pre-failure reasoning through one-sided advantage clamping. This enables the model to learn from partially successful trajectories without being affected by noisy gradients from failed rollouts.
Together, these designs enable OpenSearch-VL to learn robust long-horizon search behavior over multimodal evidence in real-world scenarios.

Experiments across multimodal deep search benchmarks show that OpenSearch-VL consistently improves over strong baselines. 
For example, compared with the Qwen3-VL-30B-A3B \citep{Qwen3-VL} agentic baseline, our model improves the average score from 47.8 to 61.6, with large gains on VDR (+13.3) \citep{vdr}, MMSearch (+24.5) \citep{mmsearch}, FVQA (+10.2) \citep{wang2017fvqa}, and InfoSeek (+16.2) \citep{chen2023can}. 
Moreover, OpenSearch-VL achieves comparable or even better performance than proprietary commercial models on several benchmarks.
In summary, our main contributions can be summarized as follows:

\begin{itemize}[leftmargin=*, noitemsep]
    \item We introduce \textbf{OpenSearch-VL}, a fully open recipe for training frontier multimodal deep search agents. 
    We will release the training data, code, and models to provide an open foundation for reproducible research on multimodal agentic search.
    \item We build the key components required for training advanced multimodal search agents, including high-quality image-grounded multi-hop training data, a diverse tool environment, and a multi-turn fatal-aware GRPO algorithm. 
    \item Extensive experiments demonstrate the effectiveness of our recipe. For example, our trained OpenSearch-VL-30B-A3B brings an average improvement of 13.8 points across 7 multimodal deep search benchmarks.
\end{itemize}

\vspace{-0.1in}
\section{Preliminaries}
\label{sec:method}

\textbf{Problem Formulation.}
Given an input image $I_0$ and a question $q$, the agent answers $q$ by interleaving reasoning with tool calls over a diverse tool set $\mathcal{T} = \mathcal{T}_v \cup \mathcal{T}_s$, where $\mathcal{T}_v$ contains \emph{visual} tools that transform or parse images and $\mathcal{T}_s$ contains \emph{retrieval} tools that query external knowledge.
At step~$l$, the model conditions on the accumulated history
\begin{equation}
  h_l = \bigl(\mathcal{I}_l,\;\, q,\;\, \mathbf{a}_{<l},\;\, \mathbf{o}_{<l}\bigr),
  \label{eq:history}
\end{equation}
where $\mathcal{I}_l$, $\mathbf{a}_{<l}$, and $\mathbf{o}_{<l}$ denote the images, actions, and observations accumulated up to step~$l$.
The interaction unfolds as a multi-turn trajectory
\begin{equation}
    \tau = \bigl\{(h_0, a_0, o_0),\; (h_1, a_1, o_1),\; \dots,\; (h_{L-1}, a_{L-1}, o_{L-1}),\; (h_L, a_L)\bigr\},
    \label{eq:trajectory}
\end{equation}
where the final step emits the answer without a subsequent observation. Following the ReAct~\citep{yao2022react} think-then-act convention, each action decomposes as $a_l = [z_l,\, c_l]$, where $z_l$ is a reasoning trace, and $c_l$ denotes a tool invocation for $l<L$ or the final response for $l=L$.

\textbf{Multimodal Observations and Active Visual Context.}
Unlike text-only formulations~\citep{jin2025searchr1trainingllmsreason}, our environment $\mathcal{E}$ returns \emph{multimodal} observations. Given a control command $c_l$, $\mathcal{E}$ deterministically routes the invocation by tool family,
\begin{equation}
    o_l \;=\; \mathcal{E}(c_l, h_l) \;\in\;
    \begin{cases}
        \mathcal{O}^{\text{img}}, & \text{if } c_l \text{ invokes } t \in \mathcal{T}_v \setminus \{\textsc{OCR}\}, \\[2pt]
        \mathcal{O}^{\text{txt}}, & \text{if } c_l \text{ invokes } t \in \mathcal{T}_s \cup \{\textsc{OCR}\},
    \end{cases}
    \label{eq:obs-routing}
\end{equation}
so that $\mathcal{O} = \mathcal{O}^{\text{img}} \cup \mathcal{O}^{\text{txt}}$. The active visual context grows monotonically as $\mathcal{I}_l = \{I_0\} \cup \{o_k : k < l,\; o_k \in \mathcal{O}^{\text{img}}\}$; historical visual observations are strictly preserved so that the policy can cross-reference multi-hop visual transformations (e.g.\ a localised \textsc{Crop} against its \textsc{SuperResolution}-enhanced counterpart). The rollout is compactly written as $\tau \sim \pi_\theta(\cdot \mid I_0, q) \otimes \mathcal{E}$, where $\otimes$ denotes the strict interleaving of policy-emitted actions and environment-returned observations.

\textbf{Trajectory Likelihood.}
The policy models the joint trajectory probability via standard autoregressive factorisation:
\begin{equation}
    \pi_\theta(\tau \mid I_0, q) \;=\; \prod_{l=0}^{L} P_\theta(a_l \mid h_l) \;=\; \prod_{l=0}^{L} P_\theta(z_l \mid h_l)\, P_\theta(c_l \mid h_l, z_l).
    \label{eq:traj-likelihood}
\end{equation}
Observations $o_l$ are excluded from the generative probability mass since they are exogenous outputs of $\mathcal{E}$; they influence the trajectory likelihood only by modulating subsequent histories $h_{l'}$ for $l' > l$. This factorisation is the object directly supervised by SFT (Eq.~\ref{eq:sft}) and the basis of the per-token importance ratio in our RL objective (Eq.~\ref{eq:our-grpo}).

\textbf{Token-level Generation Mask.}
Optimisation gradients must be restricted to tokens emitted by the policy itself. For \emph{textual observations} (originating from $\mathcal{T}_s$ and \textsc{OCR}), we define an indicator $M_{\text{gen}}(y_t)\in\{0,1\}$ with $M_{\text{gen}}(y_t)=1$ iff token $y_t$ is constituent to a generated action $a_l = [z_l,\, c_l]$, and $M_{\text{gen}}(y_t)=0$ if $y_t$ belongs to an observation span $o_l$. \emph{Image-valued observations} (from $\mathcal{T}_v\setminus\{\textsc{OCR}\}$) are injected directly into the visual backbone and inherently bypass the token-level loss. This protocol, inspired by the retrieved-token masking of Search-R1~\citep{jin2025searchr1trainingllmsreason}, underlies both the SFT objective (Eq.~\ref{eq:sft}) and the fatal-aware RL mask (Eq.~\ref{eq:fatal-mask}); textual serialisations of search results and OCR parses are characteristically noisy and structurally divergent from the policy's intrinsic generative distribution, and including them in the loss destabilises training.

\textbf{Search Tools.}
\emph{OpenSearch-VL} is equipped with a suite of tools covering three complementary functions: \textbf{retrieval} (\textsc{TextSearch}, \textsc{ImageSearch}) for gathering external evidence, \textbf{image enhancement} (\textsc{Sharpen}, \textsc{SuperResolution}, \textsc{PerspectiveCorrect}) for remedying low-quality inputs, and \textbf{attention and parsing} (\textsc{Crop}, \textsc{OCR}) for localizing and decoding fine-grained content. The suite combines lightweight offline primitives with online services backed by expert models, and is summarized in Table~\ref{tab:search-tools}. Full specifications are deferred to Appendix~\ref{apdx:tool-definition}.

\begin{table}[H]
\centering
\vspace{-5pt}
\caption{The search-oriented tool suite integrated within \textbf{OpenSearch-VL}. The suite spans three complementary functions---\emph{retrieval} for acquiring external information, \emph{image enhancement} for improving low-quality visual inputs, and \emph{attention \& parsing} for focusing on and extracting content from specific regions. Detailed specifications of each tool are provided in Appendix~\ref{apdx:tool-definition}.}
\renewcommand{\arraystretch}{1.15}
\resizebox{\linewidth}{!}{%
\begin{tabular}{llll}
\toprule
\textbf{Tool} & \textbf{Description} & \textbf{Arguments} & \textbf{Tool Output} \\
\midrule
\textsc{TextSearch}          & Web search with page reading and LLM summarization     & Query + TopK           & Query-focused passage summaries \\
\textsc{ImageSearch}         & Reverse image / visual entity search over the web      & Image + TopK           & Visual matches and related webpages \\
\textsc{Sharpen}             & Unsharp-masking based deblurring / detail enhancement  & Image + Amount         & Sharpened image \\
\textsc{SuperResolution}     & Deep super-resolution (EDSR) for low-resolution inputs & Image + Scale          & High-resolution image \\
\textsc{PerspectiveCorrect}  & Auto perspective rectification of skewed documents     & Image                  & Fronto-parallel image \\
\textsc{Crop}                & Extract a user-specified rectangular region            & Image + Coordinates    & Cropped image \\
\textsc{OCR}                 & Structured document parsing with text and layout labels & Image + Flags         & Text blocks with labels and reading order \\
\bottomrule
\end{tabular}}
\label{tab:search-tools}
\vspace{-5pt}
\end{table}

\section{Dataset Curation}
\label{sec:data}

To equip the model with robust reasoning and tool-use capabilities, we design a scalable data curation pipeline (Figure~\ref{fig:data-pipeline}) that synthesizes high-quality trajectories without manual human annotation. The pipeline proceeds in three stages---VQA construction, staged filtering and enhancement, and trajectory synthesis---yielding the final dataset used for the following stage training.

\begin{figure}[t]
  \centering
  \includegraphics[width=\linewidth]{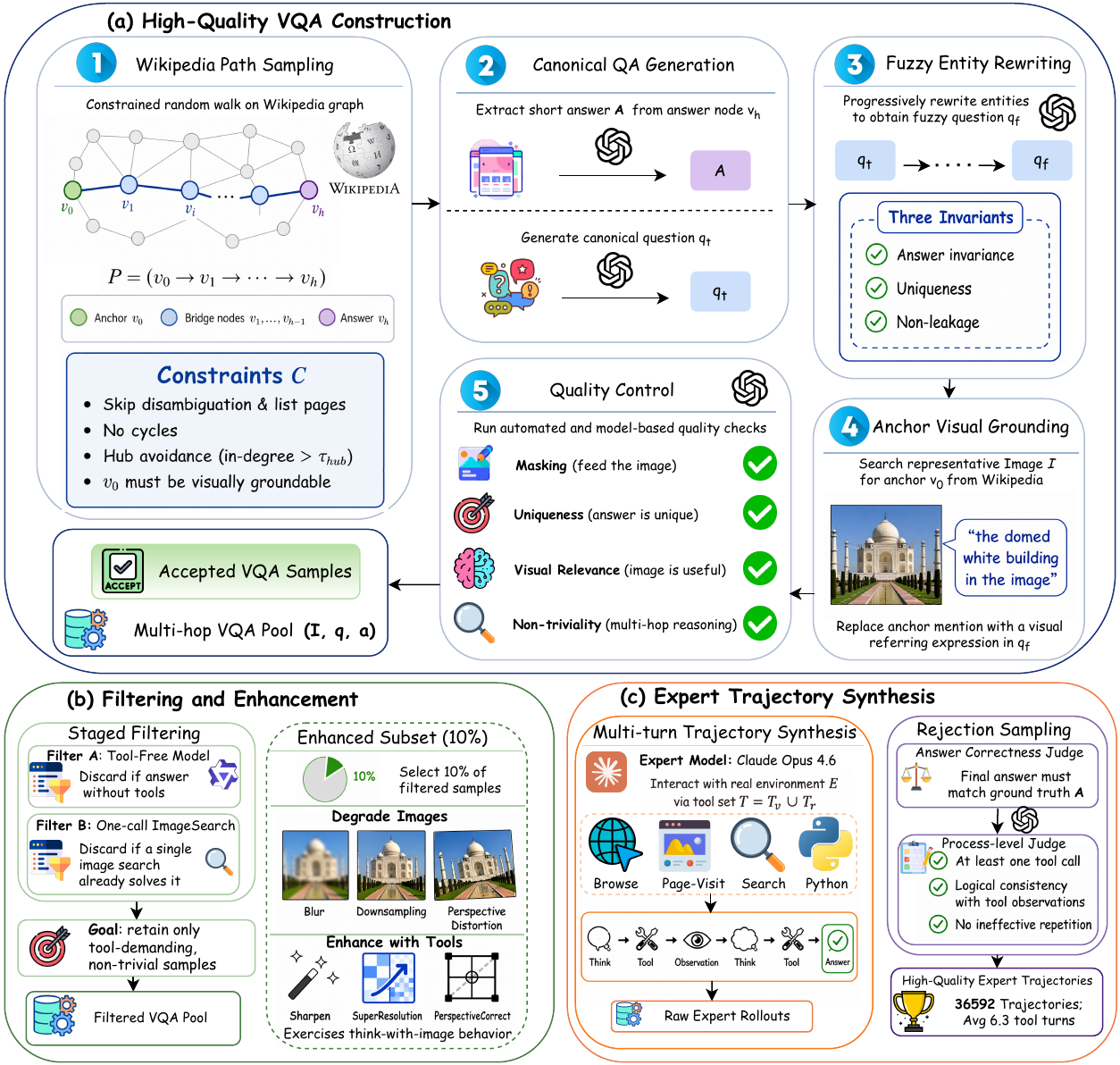}
  \vspace{-0.5cm}
    \caption{\textbf{Overview of the data curation pipeline.}
    (a) Starting from the English Wikipedia hyperlink graph, we construct high-quality multi-hop VQA instances by sampling constrained paths, generating canonical question--answer pairs, rewriting them into fuzzy questions, grounding anchor entities with representative images, and applying automated quality control.
    (b) We then perform staged filtering to retain only tool-demanding, non-trivial samples, and create an enhanced subset through image degradation and tool-based restoration to encourage think-with-image behavior.
    (c) Finally, we synthesize multi-turn expert trajectories in a real tool environment and apply rejection sampling with answer-correctness and process-level judges, yielding the final high-quality trajectories used for the following stage training.}
  \label{fig:data-pipeline}
\vspace{-15pt}

\end{figure}

\subsection{High-Quality VQA Construction}
\label{sec:vqa-construction}

A central challenge for training multimodal search agents is the supply of questions that encourage non-trivial use of the diverse tool set~$\mathcal{T}$. Directly prompting a VLM on an image tends to yield shallow, perception-level queries that can be resolved in a single forward pass~\citep{webwatcher,huang2026vision}. 
Building on this observation, we adopt a unified construction pipeline: we sample multi-hop trajectories over the Wikipedia hyperlink graph, synthesize textual QA pairs along each trajectory, and lift them into image-grounded VQA via answer-preserving fuzzy rewriting and source-anchored visual grounding. Compared with prior QA constructions~\citep{wu2025webdancer,li2025websailor,webwatcher}, our pipeline (i) assigns each node on the sampled path an explicit functional role within the reasoning chain, and (ii) deliberately decouples the visual anchor from the answer entity, thereby suppressing single-shot retrieval shortcuts.

\textbf{Wikipedia Path Sampling.}
We cast the Wikipedia~\citep{wikipedia} as a directed graph $\mathcal{G}=(\mathcal{V},\mathcal{E})$ with articles as nodes and in-article hyperlinks as edges. From a seed $v_0\in\mathcal{V}$, a constrained random walk of length $h\in\{2,3,4\}$ produces a path
\begin{equation}
  P \;=\; \bigl(v_0 \xrightarrow{\rho_1} v_1 \xrightarrow{\rho_2} \cdots \xrightarrow{\rho_h} v_h\bigr),
  \label{eq:wiki-path}
\end{equation}
where each relation $\rho_j$ is induced by the hyperlink's anchor text. The walk skips (i) disambiguation and list pages, (ii) cycles, and (iii) hub nodes whose in-degree exceeds a threshold $\tau_\text{hub}$; full thresholds, resampling heuristics, and additional filters are deferred to Appendix~\ref{apdx:vqa-running-example}. Each node on $P$ is assigned a functional role: $v_0$ is the \textbf{anchor} (visual entry point, to be replaced by a visual referring expression), $v_1,\dots,v_{h-1}$ are \textbf{bridge} nodes (intermediate entities with fuzzified names), and $v_h$ is the \textbf{answer} node (source of the target attribute). These roles govern the rewriting and grounding stages below.

We extract a short, unambiguous answer $a$ from $v_h$ and prompt GPT-4o~\citep{openai2024gpt4ocard} to synthesize a \emph{canonical} question $q_t$ that verbalizes $P$ and references $v_h$ only through the queried attribute (extraction details in Appendix~\ref{apdx:vqa-running-example}). The canonical $q_t$ is not a training target but a manipulable object for rewriting.

\textbf{Fuzzy Entity Rewriting.}
Preserving entity names in $q_t$ enables the agent to short-circuit the chain with a single retrieval~\citep{li2025websailor,huang2026vision}. We therefore progressively rewrite $q_t$ into a fuzzy counterpart $q_f$ while fixing $a$. Following the iterative style of Skywork-R1V4~\citep{zhang2025skywork}, we rewrite \emph{one entity at a time}, from the farthest bridge $v_{h-1}$ toward $v_0$: each name is replaced by a relational or attribute-based descriptor drawn from the entity's Wikipedia context, and an LLM uniqueness evaluator verifies that the substitution still resolves to the intended entity conditional on the partially rewritten question. A rewrite is accepted only when
\begin{equation}
\underbrace{a(q_f)=a(q_t)}_{\text{answer invariance}},\qquad
\underbrace{\lvert\mathcal{R}(q_f)\rvert=1}_{\text{uniqueness}},\qquad
\underbrace{\Bigl(\textstyle\bigcup_{j=0}^{h}\mathrm{aliases}(v_j)\Bigr)\cap q_f = \emptyset}_{\text{non-leakage}},
\label{eq:fuzz-invariants}
\end{equation}
where $\mathcal{R}(q_f)$ denotes the set of entities compatible with $q_f$ under the evaluator's world knowledge. We further interleave entity rewriting with occasional \emph{answer obfuscation}~\citep{huang2026vision} to avoid collapsing onto a stereotyped relational template.

\textbf{Anchor-aware Visual Grounding.}
We retrieve a representative image $I$ of the anchor $v_0$ from Wikimedia Commons or its Wikipedia infobox, filter candidates by CLIP similarity to a short textual description of $v_0$, and replace $v_0$ in $q_f$ with a visual referring expression (e.g., \emph{``the person in the image''}) to yield the final question $q$. Unlike prior QA-to-VQA conversions~\citep{webwatcher,zhang2025skywork} that ground on or near the answer entity, anchoring $v_0$ at the \emph{source} of $P$ substantially reduces single-hop shortcuts: the agent must first identify the visual anchor and then follow the intermediate textual relations before reaching $a$.

Each candidate triple $(I,q,a)$ is gated by automatic checks for masking, uniqueness, and visual relevance, generalizing the selector/examiner protocol of WebWatcher~\citep{webwatcher} (full criteria in Appendix~\ref{apdx:vqa-running-example}); non-triviality is handled jointly with the staged filtering of Sec.~\ref{sec:filtering-enhancement}. Instances passing these checks form the Wikipedia portion of our VQA pool, subsequently merged with open-source multimodal corpora before trajectory synthesis.

\subsection{Filtering and Enhancement}
\label{sec:filtering-enhancement}

Before trajectory synthesis, we consolidate the Wikipedia-derived VQA instances from Sec.~\ref{sec:vqa-construction} with three open-source multimodal corpora---LiveVQA~\citep{fu2025livevqa}, FVQA~\citep{wang2017fvqa}, and WebQA~\citep{chang2022webqamultihopmultimodalqa}---to broaden coverage across live entities, commonsense fact lookup, and open-web multi-hop reasoning. We then apply a two-stage difficulty filter using a frozen \texttt{Qwen3-VL-32B}~\citep{Qwen3-VL}: first discarding examples answerable without tools, and then discarding examples solvable with a single \textsc{ImageSearch} call. This removes samples that rely only on parametric knowledge, perceptual shortcuts, answer-coincident anchors, or one-hop bridge leakage, ensuring that retained instances genuinely require the intended visual-to-text search chain.

To further expose the agent to realistic visual imperfections, we randomly select $10\%$ of the filtered VQA pool and apply controlled degradations---blur, downsampling, and perspective distortion---paired with the corresponding enhancement tools in $\mathcal{T}_v$ (\textsc{Sharpen}, \textsc{SuperResolution}, and \textsc{PerspectiveCorrect}). This enhancement subset diversifies the training distribution and induces a \emph{think-with-image} behavior: when the input image is unreliable, the policy learns to repair the visual evidence before initiating retrieval. Together, the filtered retrieval-heavy instances and the enhancement-required subset exercise both visual restoration and evidence acquisition within the unified tool environment.

\subsection{Multi-turn Trajectory Synthesis}
\label{sec:trajectory-synthesis}

For each instance $(I, q, a)$ that survives the filters of Sec.~\ref{sec:filtering-enhancement}, we synthesize expert trajectories by rolling out \texttt{Claude Opus 4.6}~\citep{Claude_46} as the expert model against the real execution environment $\mathcal{E}$, prompted with the agent system prompt of Appendix~\ref{apdx:system-prompts} and free to invoke any tool in $\mathcal{T}$. We draw $K=5$ independent rollouts per instance, each formatted as a multi-turn ReAct~\citep{yao2022react} trajectory aligned with Eq.~\ref{eq:trajectory}. Then the raw rollouts are passed through a two-stage rejection cascade. The first stage discards any trajectory whose final answer disagrees with the ground truth $a$ (adjudicated by the same \texttt{GPT-4o}~\citep{openai2024gpt4ocard} LLM-as-judge~\citep{gu2025surveyllmasajudge} we use for $r_{\text{acc}}$, Sec.~\ref{sec:multi-turn-grpo}). The surviving trajectories are then vetted by a \texttt{GPT-5.4} process-level judge on tool-use, logical consistency between reasoning and observations, and absence of ineffective repetition, sharing the four-dimension rubric of $r_{\text{query}}$ (Sec.~\ref{sec:multi-turn-grpo}).

Applying both stages to the full rollout corpus yields $\mathbf{36{,}592}$ high-quality expert trajectories with an average of $\mathbf{6.3}$ tool-invocation turns per trajectory, which together constitute the SFT corpus consumed in Sec.~\ref{sec:sft}.

\section{Training}
\label{sec:training}

\begin{figure}[t]
  \centering
  \includegraphics[width=\linewidth]{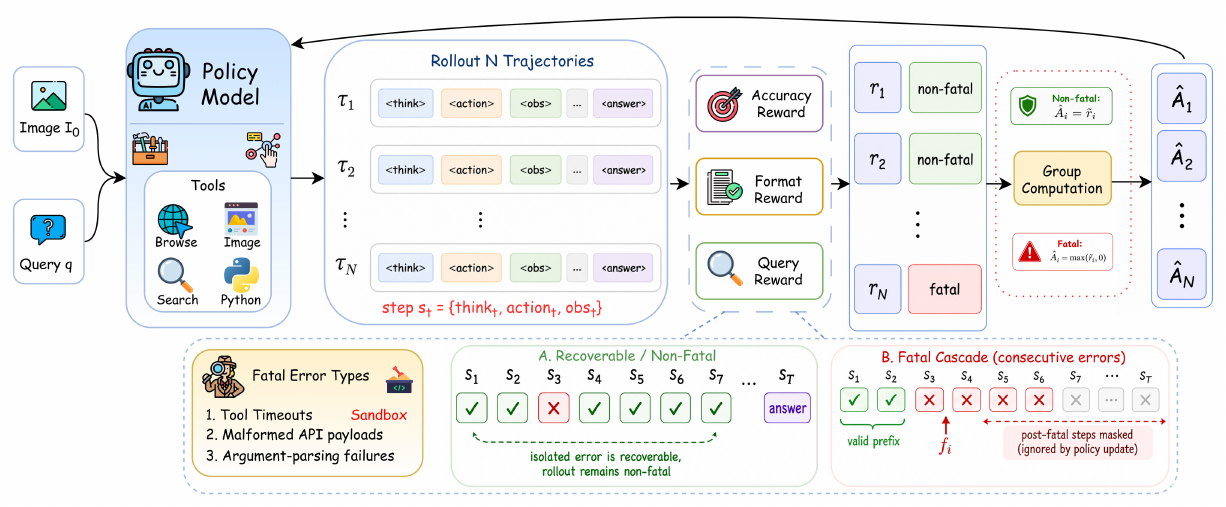}
  \vspace{-0.5cm}
  \caption{\textbf{Overview of the RL training pipeline.}
  Starting from a supervised fine-tuned model, we sample a group of multi-turn trajectories against the real environment $\mathcal{E}$. Each trajectory is evaluated by a composite reward combining final-task success ($r_{\text{acc}}$) and process-level search quality ($r_{\text{query}}$) along with a format check ($r_{\text{fmt}}$). To preserve valid reasoning in trajectories that eventually encounter fatal errors, we apply fatal-aware token masking to truncate the sequence and employ one-sided advantage clamping during policy optimization, preventing the suppression of viable early steps.}
  \label{fig:rl-pipeline}
  \vspace{-0pt}
\end{figure}

We train \textsc{OpenSearch-VL} in two sequential stages. First, we perform supervised fine-tuning (SFT) to instill fundamental reasoning and tool-use behaviors; subsequently, we apply reinforcement learning (RL) via a multi-turn, search-augmented objective (Figure~\ref{fig:rl-pipeline}) to discover more effective exploration strategies.

\subsection{Supervised Fine-Tuning}
\label{sec:sft}

We perform SFT on a curated set of $36592$ multi-turn
expert trajectories (Section~\ref{sec:data}).
Using the history $h_l$ (Eq.~\ref{eq:history}) and the action
decomposition $a_l=[z_l,\,c_l]$, by autoregressive factorisation the
step-level action probability decomposes as
\begin{equation}
  P_\theta(a_l \mid h_l)
  \;=\;
  P_\theta(z_l \mid h_l)\;
  P_\theta(c_l \mid h_l,\, z_l).
\end{equation}
Summing over all trajectories $i\in\{1,\dots,N\}$ and steps
$l\in\{1,\dots,L_i\}$, the standard SFT objective can be equivalently
written as
\begin{equation}
  \max_\theta \sum_{i=1}^{N} \sum_{l=1}^{L_i}
  \left[
    \log P_\theta\left(z_l^{(i)} \mid h_l^{(i)}\right)
    +
    \log P_\theta\left(c_l^{(i)} \mid h_l^{(i)},\, z_l^{(i)}\right)
  \right],
  \label{eq:sft}
\end{equation}
where tool observations $o_l$ enter only as conditioning context and are
excluded from the loss computation following the retrieved-token masking
strategy of \citep{jin2025searchr1trainingllmsreason}.
This provides a structured interpretation of the training signal,
showing that it jointly supervises both the reasoning trace and the
subsequent tool invocation (or terminal response) at each step.

\subsection{Multi-Turn Search Fatal-Aware GRPO}
\label{sec:multi-turn-grpo}

While SFT provides a strong initialization for tool use, it remains bounded by the coverage of the demonstration trajectories and therefore cannot discover improved search strategies through exploration. We address this limitation with reinforcement learning, building on GRPO~\citep{shao2024deepseekmath} and its search-augmented extension~\citep{jin2025searchr1trainingllmsreason}. Our setting, however, differs from prior search-only formulations in three important respects: we optimize over a multimodal environment  $\mathcal{E}$ with diverse tools rather than a text-only retriever $\mathcal{R}$; we use a composite reward that combines final-task success with process-level search quality; and we introduce fatal-aware masking together with one-sided advantage clamping to preserve useful supervision from partially successful trajectories.

During training, for each prompt $(I_0, q)$ we sample a group of $G$ multi-turn rollouts $\tau_i \sim \pi_{\theta_{\text{old}}}(\cdot \mid I_0, q) \otimes \mathcal{E}$, for $i=1,\dots,G$.

\textbf{Composite Multi-Turn Reward.}
Long-horizon tasks pose a sparse-reward challenge: outcome-only rewards miss credit for partially successful reasoning, while process-only rewards risk misalignment from the end goal. We therefore use a composite trajectory-level reward
\begin{equation}
  r(\tau) \;=\; r_{\text{fmt}}(\tau) \;\cdot\; \bigl[\,\alpha\, r_{\text{acc}}(\tau) + (1{-}\alpha)\, r_{\text{query}}(\tau)\,\bigr],
  \label{eq:reward}
\end{equation}
where $\alpha=0.8$. The composite trajectory-level reward (Eq.~\ref{eq:reward}) is structured to balance algorithmic formatting, terminal accuracy, and process-level search quality. We define each component as follows:

\begin{itemize}[leftmargin=1.4em, itemsep=4pt, topsep=2pt, parsep=0pt]
  \item \emph{Format reward $r_{\text{fmt}} \in [0,1]$.} A deterministic, algorithmic prior that enforces structural integrity. We define $r_{\text{fmt}}(\tau) = \frac{1}{L+1} \sum_{l=0}^{L} r_{\text{fmt}}^{(l)}$, where $r_{\text{fmt}}^{(l)} = 1$ iff step $l$ emits a contiguous \texttt{<think>}$\,\cdots\,$\texttt{</think>} block immediately followed by either a \texttt{<tool\_call>}$\,\cdots\,$\texttt{</tool\_call>} block (for $l < L$) or a \texttt{<response>}$\,\cdots\,$\texttt{</response>} block (for $l = L$); $r_{\text{fmt}}^{(l)} = 0$ for any structural violation, including steps that trigger tool-execution errors (e.g., malformed API arguments). By acting as a multiplicative gate in Eq.~\ref{eq:reward}, $r_{\text{fmt}}$ drives the overall return of structurally degraded trajectories toward zero.

  \item \emph{Accuracy reward $r_{\text{acc}} \in \{0,1\}$.} A terminal outcome metric assessing the fidelity of the agent's final resolution. A \texttt{GPT-4o}~\citep{openai2024gpt4ocard} judge under an LLM-as-Judge protocol~\citep{gu2025surveyllmasajudge} verifies semantic equivalence between the agent's terminal \texttt{<response>} and the ground-truth annotation, assigning $r_{\text{acc}}=1$ for a match and $0$ otherwise. \emph{Convention for truncated trajectories:} for trajectories aborted by the fatal-state condition (Eq.~\ref{eq:fatal-mask}) before emitting a terminal \texttt{<response>}, we deterministically assign $r_{\text{acc}}=0$. This is a structural guarantee rather than an evaluative judgment---in the absence of a terminal answer, correctness is strictly undefined and conservatively zeroed out---and it ensures the outcome signal remains well-defined for group-relative advantage estimation regardless of completion status.

  \item \emph{Query-quality reward $r_{\text{query}} \in [0,1]$.} A process-level signal that counteracts the inherent sparsity of $r_{\text{acc}}$ in long-horizon interactions. We use \texttt{GPT-5.4}~\citep{GPT-5}, a proprietary frontier reasoning model, as the query-quality judge to score the cumulative sequence of search queries on a continuous $[0,1]$ scale, providing dense feedback for unsuccessful trajectories with $r_{\text{acc}}=0$. The rubric covers four dimensions: (i) semantic relevance of issued queries to the initial prompt; (ii) logical progression and iterative refinement of queries across successive turns; (iii) signal-to-noise ratio within retrieved payloads; and (iv) cross-modal complementary use of image and text retrieval tools. For trajectories designated as fatal, the judge restricts its evaluation to the valid pre-fatal prefix (steps $l < f_i$), so that early-stage reasoning is credited despite subsequent collapse.
\end{itemize}

\textbf{Fatal-Aware Token Masking.}
In unconstrained multi-turn environments, agents frequently encounter ``fatal'' states---such as cascading tool-execution failures or infinite loops---after which subsequent reasoning becomes meaningless. Standard approaches either discard the entire trajectory~\citep{huang2026vision} (wasting the valid early steps) or train on it blindly (injecting noise). We introduce a \emph{fatal-aware} masking strategy to preserve the viable prefix. 

We define the \emph{fatal step index} $f_i$ for trajectory $\tau_i$ as the earliest step where $K=3$ consecutive tool-execution errors commence, with $f_i = L_i+1$ if no such cascade occurs. We then extend the observation-token (generation) mask $M_{\text{gen}}(y_{i,t})$ to additionally zero out all tokens generated \emph{after} the fatal step:
{\small
\begin{equation}
  M(y_{i,t}) \;=\; M_{\text{gen}}(y_{i,t}) \;\cdot\; \mathbb{1}\bigl[\,s(t) < f_i\,\bigr],
  \label{eq:fatal-mask}
\end{equation}
}where $s(t)$ maps token index $t$ to its step index $l$.
Crucially, the process rewards $r_{\text{fmt}}$ and $r_{\text{query}}$ are computed exclusively over the valid prefix $l < f_i$, ensuring the model is not penalised for structural collapse that occurs after the trajectory is deemed fatal.

All $G$ trajectories---including fatal ones---contribute to the standard group-normalised reward $\widetilde{r}_i = (r(\tau_i) - \mathrm{mean})/(\mathrm{std} + \delta)$ to keep group statistics unbiased. However, directly using $\widetilde{r}_i$ for fatal trajectories is pathological: a sub-mean $\widetilde{r}_i$ would push the policy gradient to suppress the \emph{viable prefix}, discouraging the valid reasoning before the error cascade. We therefore apply one-sided advantage clamping:
{\small
\begin{equation}
  \hat{A}_i \;=\;
  \begin{cases}
    \widetilde{r}_i & \text{if } f_i = L_i+1 \text{ (non-fatal)},\\[3pt]
    \max(\widetilde{r}_i,\; 0) & \text{if } f_i \le L_i \text{ (fatal)}.
  \end{cases}
  \label{eq:advantage}
\end{equation}
}This clamping ensures that the valid prefix of a fatal trajectory is only ever \emph{reinforced} if its partial reward exceeds the group mean, and otherwise receives zero gradient rather than an undeserved penalty. In this sense, it generalizes the hard-masking baseline~\citep{huang2026vision} while recovering strictly more useful learning signal.

Integrating the fatal-aware mask $M_{i,t} \equiv M(y_{i,t})$ and the clamped advantage $\hat{A}_i$, our final GRPO objective over the multimodal environment $\mathcal{E}$ is formulated as:

\begin{equation}
\mathcal{J}(\theta)=\mathbb{E}_{\begin{subarray}{l}
({I_0,}q)\sim\mathcal{D} \\
\{{\tau_i}\}_{i=1}^{G}\sim\pi_{\theta_{\text{old}}}(\cdot\mid{I_0,}q{\color{promptamber};\,\mathcal{E}})
\end{subarray}}\!\!
\left[
\frac{1}{G}\sum_{i=1}^{G}{\color{promptblue}\frac{1}{\sum_t M_{i,t}}}\sum_{t=1}^{{|\tau_i|}}{\color{promptblue}M_{i,t}}
\min\!\left(
\rho_{i,t}(\theta){\color{prompttea}\hat{A}_{i}},
\mathrm{clip}_{1-\epsilon}^{1+\epsilon}\!\left(\rho_{i,t}(\theta)\right){\color{prompttea}\hat{A}_{i}}
\right)
\right].
\label{eq:our-grpo}
\end{equation}

\noindent
where $\rho_{i,t}(\theta) = \pi_\theta(y_{i,t} \mid I_0,q,y_{i,<t};\mathcal{E}) \,/\, \pi_{\theta_{\text{old}}}(y_{i,t} \mid I_0,q,y_{i,<t};\mathcal{E})$ is the token-level importance ratio, and the standard $\beta\,D_{\mathrm{KL}}[\pi_\theta\|\pi_{\mathrm{ref}}]$ term is omitted from the display since it is identical to that of standard GRPO (Eq.~\ref{eq:grpo-objective}). Relative to search-augmented GRPO~\citep{jin2025searchr1trainingllmsreason}, the differences (color-matched to Eq.~\ref{eq:our-grpo}) are threefold: \textcolor{promptamber}{the execution environment is generalized from $\mathcal{R}$ to $\mathcal{E}$}, \textcolor{promptblue}{the generation mask $M_{\text{gen}}$ is extended to the fatal-aware mask $M$}, and \textcolor{prompttea}{the advantage is computed from the composite reward with one-sided clamping}. Additional details and derivations are provided in Appendix~\ref{apdx:grpo-detail}.
\section{Experiments}
\label{sec:experiments}

\setlength{\textfloatsep}{2pt plus 1pt minus 1pt}
\setlength{\floatsep}{2pt plus 1pt minus 1pt}
\setlength{\intextsep}{2pt plus 1pt minus 1pt}
\setlength{\dbltextfloatsep}{2pt plus 1pt minus 1pt}
\setlength{\dblfloatsep}{2pt plus 1pt minus 1pt}

We first describe our experimental setups as follows:

\textbf{Models and Benchmarks.}
Our \textsc{OpenSearch-VL} is built on three Qwen3-VL variants~\citep{Qwen3-VL}: \textit{Qwen3-VL-8B-Instruct}, \textit{Qwen3-VL-30B-A3B-Instruct}, and \textit{Qwen3-VL-32B-Instruct}. For evaluation, we use seven knowledge-intensive benchmarks from our main results:  SimpleVQA~\citep{cheng2025simplevqa}, VDR~\citep{vdr}, MMSearch~\citep{mmsearch}, LiveVQA~\citep{fu2025livevqa}, BrowseComp-VL~\citep{webwatcher}, FVQA~\citep{wang2017fvqa}, and InfoSeek~\citep{chen2023can}. Together, they cover visual entity recognition, web evidence retrieval, multi-hop reasoning, and long-tail QA.

\textbf{Baselines and Evaluation Metrics.}
We evaluate \textsc{OpenSearch-VL} against baseline types: Direct Reasoning, where the model answers from its parametric knowledge and visual perception alone; RAG Workflow, where external retrieval results are provided in-context but the reasoning remains single-pass; and Agentic Workflow, where the model autonomously interleaves reasoning with tool calls in a multimodal environment. We report Pass@1 on all seven benchmarks. Correctness is adjudicated by a \textit{GPT-4o} judge that compares the model's final response with the reference answer. To ensure fair comparison across heterogeneous answer styles, we adopt the same evaluation protocol as VDR-Bench~\citep{vdr}; the full judge prompt is provided in Figure~\ref{apdxfig:bench-judge}.

\textbf{Datasets.}
Our SFT data consists of the \textbf{36K} multi-turn trajectories synthesized by the procedure in Sec.~\ref{sec:trajectory-synthesis}. For RL training, we randomly sample \textbf{8K} examples from the VQA pool after the staged filtering and enhancement process in Sec.~\ref{sec:filtering-enhancement}, ensuring that these examples are disjoint from the VQA instances used to synthesize the SFT trajectories. 

\textbf{Implementation Details.}
\textsc{OpenSearch-VL} extends \textsc{LlamaFactory}~\citep{zheng2024llamafactory} for agentic SFT and builds on \textsc{rLLM}~\citep{rllm2025} and \textsc{VDR}~\citep{huang2026vision} for multi-turn tool-interleaved RL. All stages of \textsc{OpenSearch-VL} are trained on Nvidia H20 GPUs. Agentic SFT takes roughly 2 days for the 8B dense model and 4 days for the 30B-A3B MoE model on 256 H20s (32 nodes $\times$ 8 GPUs); the subsequent multi-turn fatal-aware GRPO stage runs for approximately 200 optimization steps over 10 days on 64 H20s (8 nodes $\times$ 8 GPUs). The complete per-stage hyperparameter configurations are listed in the following subsections (SFT in Table~\ref{tab:sft-config}; RL in Table~\ref{tab:rl-config}). More details are reported in Appendix~\ref{apdx:implementation-details}.

\subsection{Main Results}

\begin{table}[!t]
    \setlength{\abovecaptionskip}{2pt}
    \setlength{\belowcaptionskip}{2pt}
    \caption{
    Performance on multimodal knowledge-intensive QA and web-search benchmarks.
    \textbf{Bold} and \underline{underline} mark the best and second-best score in each column.
    }
    \centering
    \footnotesize
    \setlength{\tabcolsep}{4.5pt}
    \resizebox{\textwidth}{!}{
    \begin{tabular}{l ccccccc c}
    \toprule
    \textbf{Model} & SimpleVQA & VDR & MMSearch & LiveVQA & BrowseComp-VL & FVQA & InfoSeek & \textbf{Avg} \\
    \midrule
    \rowcolor{blue!5}
    \multicolumn{9}{c}{\textbf{Direct Reasoning}} \\
    GPT-4o ~\citep{openai2024gpt4ocard}           & 51.7 & 1.7  & 18.7 & 28.1 & 5.5  & 48.0 & 52.9 & 29.5 \\
    GPT-5 ~\citep{GPT-5}            & \underline{61.6} & \textbf{9.8}  & \underline{35.1} & 44.4 & \textbf{48.6} & \underline{54.4} & \textbf{61.7} & \underline{45.1} \\
    Gemini-2.5-Flash~\citep{comanici2025gemini}  & 57.9 & 6.2  & 30.4 & \underline{51.0} & 37.1 & 47.7 & 44.1   & 39.2 \\
    Gemini-2.5-Pro~\citep{comanici2025gemini}    & \textbf{63.0} & \underline{8.0}  & \textbf{39.8} & \textbf{60.3} & \underline{43.1} & \textbf{60.7} & 46.9   & \textbf{46.0} \\
    Claude-4-Sonnet~\citep{Claude_4}   & 50.9 & 2.0  & 18.7 & 38.5 & 29.3 & 35.3 & \underline{57.3}   & 33.1 \\
    Claude-3.7-Sonnet~\citep{Claude_37} & 42.7 & 4.6  & 21.1 & 38.0 & 32.3 & 36.7 &  54.8   & 32.9 \\
    Qwen3-VL-8B ~\citep{Qwen3-VL} & 47.1 & 2.8  & 11.7 & 23.1 & 24.1 & 24.2 & 23.1 & 22.3 \\
    Qwen3-VL-30B-A3B ~\citep{Qwen3-VL} & 53.2   & 3.8  & 18.7 & 42.7 & 29.6 & 34.7 & 26.4 & 29.9 \\
    Qwen3-VL-32B ~\citep{Qwen3-VL} & 58.0 & 4.1   & 19.8 & 45.5 & 30.8 & 34.1 & 28.8 & 31.6 \\
    \midrule
    \rowcolor{blue!5}
    \multicolumn{9}{c}{\textbf{RAG Workflow}} \\
    GPT-4o ~\citep{openai2024gpt4ocard} & \textbf{63.6} & 4.5  & \underline{49.1} & \underline{40.1} & 13.4 & \textbf{66.3} & 59.5 & \underline{42.4} \\
    GPT-5  ~\citep{GPT-5} & 55.9 & \textbf{22.3}   & \textbf{52.6} & \textbf{56.0} & \textbf{54.9} & \underline{62.6} & \textbf{70.6} & \textbf{53.6} \\
    Claude-3.7-Sonnet~\citep{Claude_37} & 59.3 & \underline{11.3}   & 32.7 & 30.3 & 10.0 & 59.1   & \underline{60.2}   & 37.6 \\
    Qwen3-VL-8B ~\citep{Qwen3-VL}       & \underline{62.3} & 7.3   & 47.3 & 39.3 & \underline{29.3} & 53.6 & 46.1 & 40.7 \\
    \midrule
    \rowcolor{blue!5}
    \multicolumn{9}{c}{\textbf{Agentic Workflow}} \\
    DeepMMSearch-R1-7B ~\citep{deepmmsearch-r1} & 55.8 & --   & --   & --   & --   & --   & 47.5 & -- \\
    Visual-ARFT-7B ~\citep{liu2025visualarft}                   & 42.4 & 3.3   & 34.5 & 25.4 & 16.5 & 41.7 & 37.9 & 28.8 \\
    MMSearch-R1-7B\citep{mmsearch-r1} & 57.4 & 2.9   & 53.8 & 48.4 & 20.9 & 58.4 & 55.1 & 42.4 \\
    DeepEyes-v2-7B~\citep{deepeyesv2} & 59.4 & 7.8   & 63.7 & --   & --   & 60.6 & 51.1 & -- \\
    WebWatcher-7B~\citep{webwatcher}  & 54.3 & 10.3   & 49.1 & 51.2 & 21.2 & --   & --   & -- \\
    Qwen3-VL-8B~\citep{Qwen3-VL}  & 52.0 & 17.0 & 37.4 & 50.6 & 27.9 & 58.7 & 50.3 & 42.0 \\
    SenseNova-MARS-8B~\citep{chng2025sensenova} & \underline{61.7} & \underline{19.4}   & \textbf{67.4} & \underline{56.2} & \underline{35.1} & \underline{67.1} & \underline{61.7} & \underline{52.7} \\
    \rowcolor{purple!15}
    \textbf{OpenSearch-VL-8B(Ours)}         & \textbf{71.6}   & \textbf{20.8}   & \underline{64.5}   & \textbf{59.6}   & \textbf{37.6}   & \textbf{71.5}   & \textbf{70.2}   & \textbf{56.6} \\
    \midrule
    Qwen3-VL-30B-A3B~\citep{Qwen3-VL} & \underline{55.1}   & \underline{20.2} & \underline{44.2} & \underline{62.0} & \underline{34.1} & \underline{63.0} & \underline{56.2}   & \underline{47.8} \\
    \rowcolor{purple!15}
    \textbf{OpenSearch-VL-30B-A3B(Ours)}    & \textbf{74.9}   & \textbf{33.5}   & \textbf{68.7}   & \textbf{67.4}   & \textbf{41.1}   & \textbf{73.2}   & \textbf{72.4}   & \textbf{61.6} \\
    \midrule
    Qwen3-VL-32B~\citep{Qwen3-VL} & 58.7 & \underline{23.1}  & 53.9 & 45.5 & \underline{35.1} & \underline{61.2} & \underline{58.5} & \underline{48.0} \\
    WebWatcher-32B~\cite{webwatcher} & \underline{59.0} & --   & \underline{55.3} & \underline{58.7} & 27.0 & --   & --   & -- \\
    \rowcolor{purple!15}
    \textbf{OpenSearch-VL-32B(Ours)}        &  \textbf{76.2}  & \textbf{33.8}   & \textbf{72.3}   & \textbf{70.5}   & \textbf{43.8}   & \textbf{74.7}   & \textbf{74.8}   & \textbf{63.7} \\
    \bottomrule
    \end{tabular}
    }
    \label{tab:search-qa-main}
    \vspace{0.6cm}
\end{table}

Table~\ref{tab:search-qa-main} reports the results on seven multimodal knowledge-intensive QA and web-search benchmarks. \textsc{OpenSearch-VL} exhibits a clear advantage over both direct-reasoning and RAG baselines across all scales, underscoring the necessity of an agentic loop for complex multimodal queries. Among 8B-scale agents, \textit{OpenSearch-VL-8B} achieves the best average score of \textbf{56.6}, surpassing the previous strongest open 8B agent, \textit{SenseNova-MARS-8B}, by \textbf{3.9} points on average. At larger scales, \textit{OpenSearch-VL-30B-A3B} and \textit{OpenSearch-VL-32B} further improve the average score to \textbf{61.6} and \textbf{63.7}, respectively; notably, \textit{OpenSearch-VL-32B} outperforms strong proprietary direct-reasoning models such as \textit{Gemini-2.5-Pro} and substantially exceeds the corresponding \textit{Qwen3-VL} agentic baselines. These results demonstrate that our training recipe scales effectively from 8B to 32B and yields strong gains on both search-heavy and visually grounded benchmarks.

\subsection{Ablation Study}

We conduct ablations on the Qwen3-VL-8B model to validate the two design choices that define \textsc{OpenSearch-VL}: the data synthesis pipeline that produces tool-demanding multimodal trajectories, and the fatal-aware RL recipe that improves the policy beyond offline imitation.

\providecommand{\dup}[1]{\textcolor{red!75!black}{(+#1)}}
\providecommand{\ddn}[1]{\textcolor{green!55!black}{(-#1)}}

\begin{table}[H]
    \setlength{\abovecaptionskip}{2pt}
    \setlength{\belowcaptionskip}{3pt}
    \caption{\textbf{Ablation studies on the SFT data pipeline and RL training recipe.} Deltas in the top panel are measured relative to the full pipeline; deltas in the bottom panel are measured relative to Vanilla GRPO~\citep{jin2025searchr1trainingllmsreason}.}
    \label{tab:data-pipeline-ablation}
    \label{tab:training-recipe-ablation}
    \centering
    \scriptsize
    \setlength{\tabcolsep}{4.5pt}
    \renewcommand{\arraystretch}{0.92}
    \textbf{(a) SFT data pipeline ablation}\\[-0.2em]
    \begin{tabularx}{0.85\linewidth}{l CCCC}
    \toprule
    \textbf{Settings} & SimpleVQA & InfoSeek & FVQA & \textbf{Avg.} \\
    \midrule
    \rowcolor{purple!15}
    Full pipeline \textbf{(Ours)}      & \textbf{66.1}            & \textbf{62.4}            & \textbf{65.3}             & \textbf{64.6} \\
    \quad w/o source-anchor grounding  & 53.6\ddn{12.5}           & 54.5\ddn{7.9}            & 51.2\ddn{14.1}            & 53.1\ddn{11.5} \\
    \quad w/o fuzzy entity rewriting   & 51.7\ddn{14.4}           & 56.4\ddn{6.0}            & 54.7\ddn{10.6}            & 54.3\ddn{10.3} \\
    \quad w/o staged filtering         & 57.6\ddn{8.5}            & 55.2\ddn{7.2}            & 56.3\ddn{9.0}             & 56.4\ddn{8.2}  \\
    \quad w/o enhancement subset       & 64.9\ddn{1.2}            & 61.7\ddn{0.7}            & 63.2\ddn{2.1}             & 63.3\ddn{1.3}  \\
    \bottomrule
    \end{tabularx}

    \vspace{0.35em}
    \textbf{(b) RL recipe ablation}\\[-0.2em]
    \begin{tabularx}{0.85\linewidth}{l CCCC}
    \toprule
    \textbf{Method} & SimpleVQA & InfoSeek & FVQA & \textbf{Avg.} \\
    \midrule
    Qwen3-VL-8B                                             & 52.0 & 50.3 & 58.7 & 53.7 \\
    \midrule
    Qwen3-VL-8B + SFT only                                  & 66.1 & 62.4 & 65.3 & 64.6 \\
    \quad + Vanilla GRPO                                    & 68.8 & 66.5 & 67.4 & 67.6 \\
    \quad + GRPO w/ Hard Masking                            & 68.3\ddn{0.5} & 67.9\dup{1.4} & 66.9\ddn{0.5} & 67.7\dup{0.1} \\
    \quad + GRPO w/ Fatal Masking only                      & 69.7\dup{0.9} & 68.3\dup{1.8} & 69.2\dup{1.8} & 69.1\dup{1.5} \\
    \rowcolor{purple!15}
    \quad + Fatal Masking + One-sided Clamp \textbf{(Ours)} & \textbf{71.6}\dup{2.8} & \textbf{72.4}\dup{5.9} & \textbf{71.5}\dup{4.1} & \textbf{71.8}\dup{4.2} \\
    \bottomrule
    \end{tabularx}
\end{table}

\textbf{Data Pipeline Ablation.} 
Table~\ref{tab:data-pipeline-ablation} shows that each stage of our data synthesis pipeline contributes to the final performance. The full pipeline achieves the best average score of \textbf{64.6}, while removing source-anchor grounding, fuzzy entity rewriting, or staged filtering leads to large drops of 11.5, 10.3, and 8.2 points, respectively. These results indicate that effective training data must both prevent shortcut retrieval and preserve genuinely tool-demanding queries. Removing the enhancement subset causes a smaller but consistent decline ($64.6\rightarrow63.3$ Avg.), suggesting that image-restoration trajectories mainly improve robustness rather than driving the core gains.

\textbf{Training Recipe Ablation.}
Table~\ref{tab:training-recipe-ablation} studies the effect of our RL recipe after the same Qwen3-VL-8B SFT initialization, using Search-R1-style vanilla GRPO~\citep{jin2025searchr1trainingllmsreason} as the RL baseline. SFT improves the base model from 53.7 to 64.6 average accuracy, and vanilla GRPO further raises it to 67.6, showing the benefit of online exploration. The way fatal trajectories are handled is crucial: the hard-masking strategy of Vision-DeepResearch~\citep{huang2026vision} brings almost no gain over vanilla GRPO~\citep{jin2025searchr1trainingllmsreason, guo2025deepseek}($67.6\rightarrow67.7$), while fatal masking improves the average to 69.1 by preserving valid pre-failure reasoning. Our full method with one-sided advantage clamping achieves the best score on every benchmark and reaches \textbf{71.8} average accuracy, a 4.2-point gain over vanilla GRPO. The training curves in Fig.~\ref{fig:turn-acc-curves} are consistent with this result: fatal-aware GRPO sustains longer tool-use trajectories while achieving higher batch accuracy, indicating that it encourages productive exploration rather than prematurely suppressing difficult rollouts.

\begin{figure}[H]
    \centering
    \includegraphics[width=0.84\linewidth]{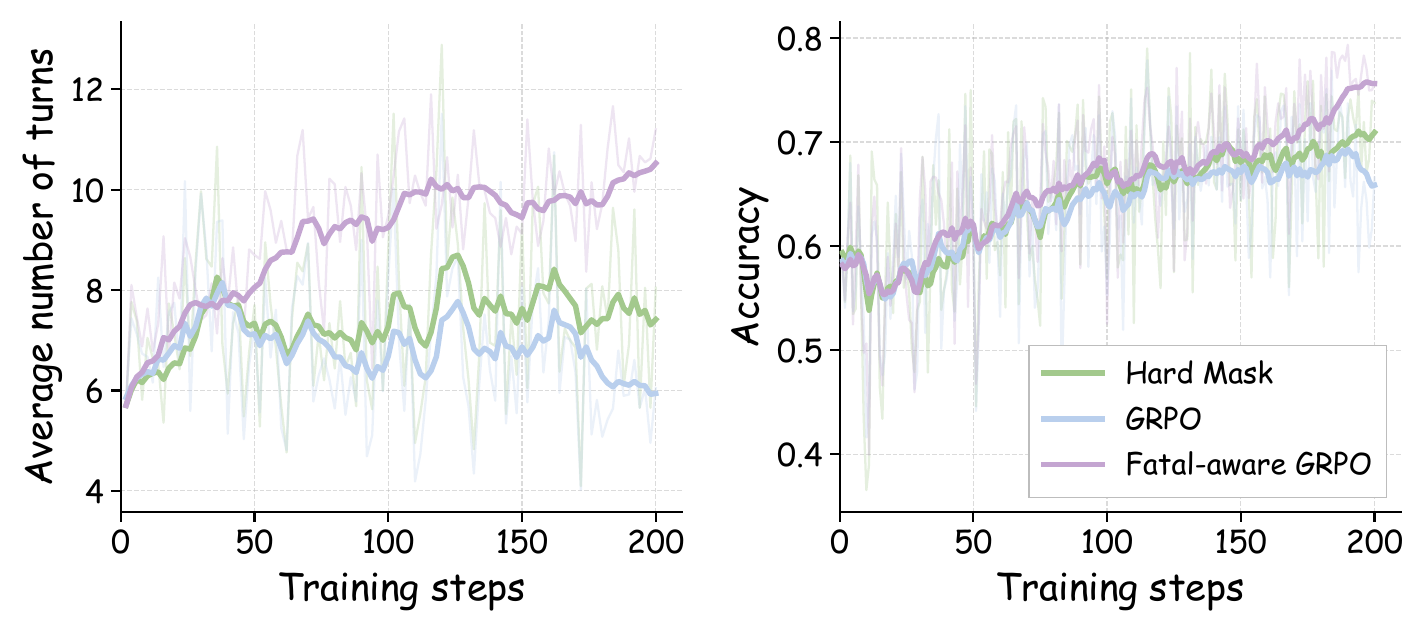}
    \caption{Training dynamics over the RL phase. \textbf{Left:} average
    number of turns per rollout. \textbf{Right:} batch-level accuracy.
    Fatal-aware GRPO sustains a higher number of turns \emph{and} reaches a
    higher accuracy than vanilla GRPO~\citep{jin2025searchr1trainingllmsreason, guo2025deepseek} and the Hard-Mask~\citep{huang2026vision} baseline.}
    \label{fig:turn-acc-curves}
\end{figure}

\begin{figure}[!t]
  \centering
  \includegraphics[width=0.92\linewidth]{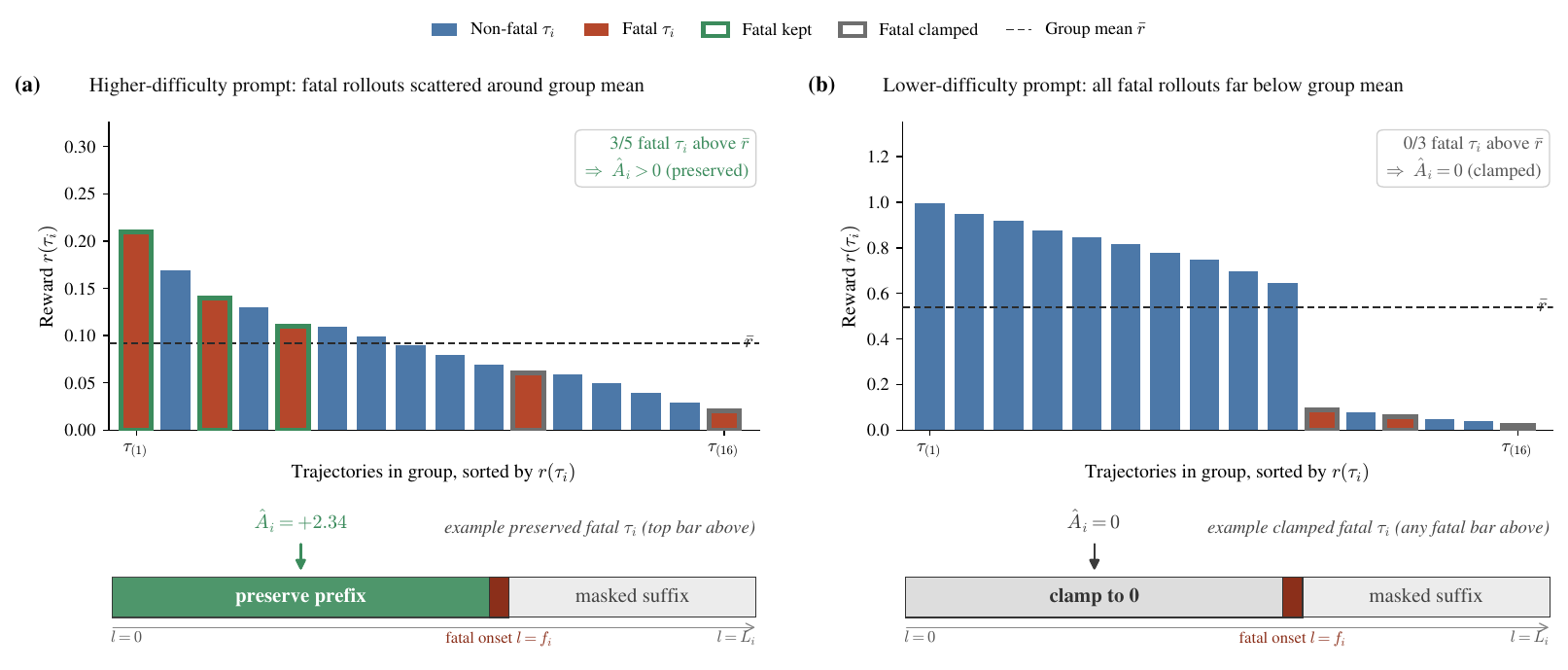}
  \caption{\small%
    \textbf{Fatal-aware masking with one-sided clamping on two illustrative rollout groups.}
    Each panel shows a group of $G{=}16$ rollouts (non-fatal in blue, fatal in brick red) together with a cartoon of the resulting token-level update on one representative fatal trajectory $\tau_i$. Bar outlines mark fatal rollouts whose pre-clamp score $\widetilde{r}_i$ exceeds the group mean $\bar{r}$ (green, preserved) or falls below it (dark grey, clamped).
    \textbf{(a) Higher-difficulty prompt.} When $\bar{r}$ is low, several fatal prefixes already beat it: $\widetilde{r}_i > 0$ and $\hat{A}_i = \widetilde{r}_i$, so gradients flow only through the viable prefix (tokens up to the fatal onset $f_i$) while the post-failure suffix is hard-masked---the partial reasoning that reached the failure point is \emph{reinforced} even though the trajectory itself is unsolvable.
    \textbf{(b) Lower-difficulty prompt.} When most rollouts succeed, every fatal reward falls far below $\bar{r}$ and one-sided clamping sets $\hat{A}_i = \max(\widetilde{r}_i,0) = 0$, degenerating to pure hard-masking and avoiding any suppression of the possibly-valid prefix.%
  }
  \label{fig:fatal-clamping-cases}
\end{figure}

\paragraph{Empirical Visualization of the Two Cases.}
Figure~\ref{fig:fatal-clamping-cases} visualizes how $\hat{A}_i = \max(\widetilde{r}_i, 0)$ interacts with per-group difficulty on two representative rollout groups. In the higher-difficulty case, most fatal trajectories still emit a coherent prefix before hitting a tool error; a fraction of them beat the group mean and contribute a positive gradient to the prefix tokens only. In the lower-difficulty case, every fatal trajectory is dominated by the successful non-fatal rollouts, and clamping to zero prevents a noisy negative signal from pushing the policy \emph{away} from what may actually be a valid prefix.
Figure~\ref{fig:fatal-clamping-aggregate} shows that this split is not anecdotal: aggregated over $10{,}000$ groups, the vast majority of fatal trajectories sit on the negative side of the pre-clamp score and are safely zeroed out, while the small preserved tail is distributionally close to the right mode of the non-fatal reference. One-sided clamping thus recovers a principled credit signal from failed tool-call trajectories without amplifying their inherent noise.

\section{Related Work}
\label{sec:related_work}

\subsection{Multimodal Agentic Search}
The integration of active search has reframed LLMs from static knowledge bases into agentic reasoners capable of dynamic information retrieval. Search-R1~\citep{jin2025searchr1trainingllmsreason} crystallized this shift by using RL to incentivize autonomous, multi-turn querying within the reasoning chain. This paradigm has since migrated to multimodal domains; MMSearch-R1~\citep{mmsearch-r1} and Vision-DeepResearch~\citep{huang2026vision, deepmmsearch_r1} embed visual retrieval into agent pipelines, while recent efforts~\citep{zhang2025skywork, yao2026mm,chen2026unify,feng2026gen, li2025websailornavigatingsuperhumanreasoning} attempt to unify image manipulation with web search. However, a common but fragile assumption in these works is the availability of pristine visual inputs. In practice, when agents encounter degraded or text-dense real-world images, the "search-only" approach fails as retrieval cannot fix fundamentally broken visual evidence.

\subsection{Visual Perception and Retrieval}
The performance of multimodal agents is often bottlenecked not by retrieval logic, but by the fidelity of initial perception~\cite{zhang2026fix,wei2025perception}. While RAG frameworks like VisRAG~\citep{yu2024visrag} emphasize preserving visual structure, they treat the model as a passive observer that must "make do" with whatever is retrieved. Even as tool-augmented agents~\citep{song2026adareasonerdynamictoolorchestration, deepeyesv2, webwatcher} introduce dynamic orchestration, their toolsets remain largely homogeneous and scenario-bound. We argue that robust reasoning requires \emph{active perception}: the agent must not only search but also intervene—autonomously invoking tools like super-resolution or specialized OCR to remediate visual noise before attempting to reason over it.

\subsection{Reinforcement Learning for Agentic Reasoning}
Training agents to operate in long-horizon, multi-tool environments poses substantial challenges for standard RL. While GRPO~\citep{shao2024deepseekmath} has shown strong effectiveness in aligning reasoning trajectories for language models~\citep{guo2025deepseek,chen2025advancing,chen2025ares,feng2025video, feng2025onethinker, hu2026openvlthinkerv2generalistmultimodalreasoning}, applying it to agentic rollouts with diverse tool interactions remains non-trivial. The primary challenge lies in \textit{cascading failures}: an early tool error renders the rest of the trajectory incoherent~\citep{vuddanti2025paladin, zhu2025where, zheng2026deepeyesincentivizingthinkingimages,dong2025agenticreinforcedpolicyoptimization}, yet these post-failure tokens still inject noise into the policy gradient~\citep{REGRPO2025, deng2026group, CoRPO2025}. Furthermore, standard group-normalization tends to penalize the valid reasoning prefixes of partially successful rollouts~\citep{huang2026vision}. To counter this, we introduce a fatal-aware RL objective that prunes learning signals from post-failure states and uses one-sided advantage clamping to protect the gradients of constructive reasoning steps.

\vspace{-0.1cm}
\section{Conclusion}
\label{sec:conclusion}
\vspace{-0.1cm}

We present \textbf{OpenSearch-VL}, a fully open recipe for training multimodal deep search agents with agentic reinforcement learning. Our recipe combines a Wikipedia-based data curation pipeline that mitigates one-step retrieval shortcuts and produces two high-quality datasets: \emph{SearchVL-SFT-36k} \& \emph{SearchVL-RL-8k}; a diverse tool environment spanning retrieval, image enhancement, and attention-and-parsing tools; and a multi-turn fatal-aware GRPO algorithm that preserves useful pre-failure reasoning through one-sided advantage clamping. Based on this recipe, OpenSearch-VL achieves over 10-point average gains across seven multimodal deep search benchmarks, with competitive performance on representative tasks such as VDR compared with strong proprietary reasoning models. We will release our data, code, models, and training recipe, with the aim to lower the reproducibility barrier and provide an open foundation for future research on multimodal deep search agents.

\section*{Limitations and Future Work}
\label{apdx:limitations}

A non-trivial fraction of training instability traces to the external tool environment $\mathcal{E}$---including search ranking drift, fetch failures, and occasional summarization hallucinations in \textsc{TextSearch} and \textsc{ImageSearch}---which inflates reward variance and motivates future work on on-policy reliability estimation. Furthermore, our composite reward (Eq.~\ref{eq:reward}) relies on proprietary GPT-4o judges, which are costly, version-dependent, and currently score only textual queries while ignoring intermediate visual operations (e.g., \textsc{Crop}); replacing these with open process reward models covering the full visual action space $\mathcal{T}_v$ remains a natural next step. Finally, exact numerical reproducibility is challenged by the reliance on these externally hosted APIs (e.g., Serper, PaddleX \textsc{OCR}) and the prohibitive cost of reporting multi-seed error bars for large-scale evaluations (Table~\ref{tab:search-qa-main}). To mitigate these constraints and support open research, we will release our complete datasets (SearchVL-SFT-36k / SearchVL-RL-8k), model checkpoints, and training code under permissive licenses.

\clearpage

{
    \bibliographystyle{plainnat}
    \bibliography{ref}
}

\newpage

\appendix

\section*{Appendix}
\label{sec:appendix}

\addcontentsline{toc}{part}{Appendix}

\startcontents[appendix]
\begingroup
\hypersetup{linkcolor=promptblue}
\printcontents[appendix]{}{1}{\section*{Appendix Contents}}
\endgroup

\section{Preliminary of Reinforcement Learning}
\label{apdx:rl}

This section details the reinforcement learning (RL) preliminaries underpinning our multi-turn training objective. We first review the standard Group Relative Policy Optimization (GRPO) algorithm~\citep{shao2024deepseekmath}, which operates on single-turn generations. Subsequently, we describe its extension to search-augmented rollouts, as introduced by Search-R1~\citep{jin2025searchr1trainingllmsreason}, wherein the policy interleaves generation with calls to an external search engine. These formulations serve as the direct precursors to the multi-turn, multi-tool objective employed by \textsc{OpenSearch-VL} (see Sec.~\ref{sec:multi-turn-grpo}).

\subsection{Standard GRPO}

GRPO~\citep{shao2024deepseekmath} is an actor-critic variant of Proximal Policy Optimization (PPO)~\citep{schulman2017proximal} that eliminates the reliance on a learned value function. Instead of maintaining a separate critic model to estimate a baseline, GRPO derives its advantage estimate by comparing multiple responses sampled for the same prompt, collectively defining a \emph{group}. This parameter-efficient design is particularly advantageous in settings where only a scalar outcome-level reward is available, reducing memory overhead and mitigating the instability of training a dense token-level value function.

Formally, given a prompt distribution $\mathcal{D}$ and a reference policy $\pi_{\theta_{\text{old}}}$, GRPO samples a group of $G$ candidate responses $\{o_1, o_2, \dots, o_G\} \sim \pi_{\theta_{\text{old}}}(\cdot \mid q)$ for each prompt $q \sim \mathcal{D}$. Each response is evaluated via a reward function $r_i = r(q, o_i)$. To compute the advantages, GRPO standardizes these rewards within the local group:
\begin{equation}
    \hat{A}_{i,t} \;=\; \widetilde{r}_i \;=\; \frac{r_i - \mathrm{mean}(\{r_j\}_{j=1}^G)}{\mathrm{std}(\{r_j\}_{j=1}^G)},
    \quad \text{for all tokens } t = 1, \dots, |o_i|.
    \label{eq:grpo-advantage}
\end{equation}
By this formulation, all constituent tokens of a given response $o_i$ are assigned a uniform advantage scalar. The policy parameters $\theta$ are subsequently optimized by maximizing the clipped surrogate objective:
\begin{equation}
\begin{aligned}
    \mathcal{J}_{\text{GRPO}}(\theta)
    \;=\;&\;\; \mathbb{E}_{q \sim \mathcal{D},\, \{o_i\}_{i=1}^G \sim \pi_{\theta_{\text{old}}}(\cdot \mid q)} \Bigg[
    \frac{1}{G} \sum_{i=1}^{G} \frac{1}{|o_i|} \sum_{t=1}^{|o_i|} \bigg\{ \\
    &\quad\; \min\!\Big( \rho_{i,t}(\theta)\,\hat{A}_{i,t},\;\;
    \mathrm{clip}\bigl(\rho_{i,t}(\theta),\,1{-}\epsilon,\,1{+}\epsilon\bigr)\,\hat{A}_{i,t} \Big) \\
    &\quad\; -\; \beta\, \mathbb{D}_{\text{KL}}\!\bigl[\pi_\theta \,\|\, \pi_{\text{ref}}\bigr] \bigg\} \Bigg],
\end{aligned}
\label{eq:grpo-objective}
\end{equation}
where the importance sampling ratio is defined as
\begin{equation}
    \rho_{i,t}(\theta) \;=\; \frac{\pi_\theta(o_{i,t} \mid q,\, o_{i,<t})}{\pi_{\theta_{\text{old}}}(o_{i,t} \mid q,\, o_{i,<t})}.
\end{equation}
Here, $\epsilon$ represents the probability ratio clipping hyperparameter, and $\beta$ modulates the Kullback--Leibler (KL) divergence penalty against a fixed reference policy $\pi_{\text{ref}}$. In contrast to standard PPO---which subsumes the KL penalty directly into the reward signal and necessitates a parameterized value network to approximate $\hat{A}_{i,t}$~\citep{schulman2017proximal}---GRPO enforces the KL regularization explicitly within the loss landscape. By deriving $\hat{A}_{i,t}$ strictly from empirical group statistics (Eq.~\ref{eq:grpo-advantage}), GRPO naturally aligns with the comparative structure of preference-based reward modeling.

\subsection{GRPO with Search Engine}

Search-R1~\citep{jin2025searchr1trainingllmsreason} extends the GRPO framework from unimodal, single-turn generation to an interleaved, search-augmented generative process. Rather than sampling an isolated response $o_i \sim \pi_{\theta_{\text{old}}}(\cdot \mid q)$, Search-R1 generates a multi-step rollout comprising both policy-emitted tokens and external evidence retrieved from a search engine $\mathcal{R}$. This search-augmented sampling distribution is denoted as:
\begin{equation}
    o_i \;\sim\; \pi_{\theta_{\text{old}}}(\cdot \mid q; \mathcal{R}) \;=\; \pi_{\theta_{\text{old}}}(\cdot \mid q) \,\bigotimes\, \mathcal{R},
\end{equation}
where $\bigotimes$ represents the interleaving operator: whenever the policy emits a search command, the retrieved documents are appended to the conditioning context, and the policy resumes generation conditioned on this expanded prefix.

To prevent the optimizer from erroneously updating parameters based on exogenous environmental tokens, Search-R1 utilizes a masking function that filters out non-generated tokens. Defining $M_{\text{gen}}(y_{i,t}) \in \{0, 1\}$ as an indicator variable where $M_{\text{gen}}(y_{i,t}) = 1$ if $y_{i,t}$ is explicitly generated by the policy, the corresponding Search-R1 GRPO objective becomes:
\begin{equation}
\begin{aligned}
    \mathcal{J}_{\text{GRPO}}^{\mathcal{R}}(\theta)
    \;=\;&\;\; \mathbb{E}_{q \sim \mathcal{D},\, \{o_i\}_{i=1}^{G} \sim \pi_{\theta_{\text{old}}}(\cdot\mid q;\,\mathcal{R})}
    \Bigg[
    \frac{1}{G} \sum_{i=1}^{G}
    \frac{1}{\sum_{t=1}^{|o_i|} M_{\text{gen}}(y_{i,t})}
    \sum_{\substack{t=1 \\ M_{\text{gen}}(y_{i,t})=1}}^{|o_i|} \bigg\{ \\
    &\quad\; \min\!\Big( \rho_{i,t}^{\mathcal{R}}(\theta)\,\hat{A}_{i,t},\;\;
    \mathrm{clip}\bigl(\rho_{i,t}^{\mathcal{R}}(\theta),\,1{-}\epsilon,\,1{+}\epsilon\bigr)\,\hat{A}_{i,t} \Big) \\
    &\quad\; -\; \beta\, \mathbb{D}_{\text{KL}}\!\bigl[\pi_\theta \,\|\, \pi_{\text{ref}}\bigr] \bigg\} \Bigg],
\end{aligned}
\label{eq:grpo-search}
\end{equation}
featuring the environment-conditioned importance ratio:
\begin{equation}
    \rho_{i,t}^{\mathcal{R}}(\theta) \;=\; \frac{\pi_\theta(y_{i,t} \mid q,\, y_{i,<t};\, \mathcal{R})}{\pi_{\theta_{\text{old}}}(y_{i,t} \mid q,\, y_{i,<t};\, \mathcal{R})}.
\end{equation}
Relative to standard GRPO (Eq.~\ref{eq:grpo-objective}), this adaptation introduces two pivotal modifications. First, the expectation is calculated with respect to the interleaved distribution $\pi_{\theta_{\text{old}}}(\cdot \mid q; \mathcal{R})$, ensuring the causal conditioning history incorporates all previously retrieved external evidence. Second, the objective is normalized strictly by the count of generated tokens, $\sum_t M_{\text{gen}}(y_{i,t})$, restricting gradient calculations to policy-authored positions. Equivalently masking the KL divergence ensures the model is not penalized for distributional divergence over environmental observations. These mechanisms adapt GRPO to tool-in-the-loop architectures and establish the theoretical foundation for our multimodal environment $\mathcal{E}$ formulation (Sec.~\ref{sec:multi-turn-grpo}).

\section{Multi-Turn Search Fatal-Aware GRPO Details}
\label{apdx:grpo-detail}

Building upon the foundations established in Appendix~\ref{apdx:rl}, this section details the mechanics of our multi-turn, search-augmented RL objective introduced in Sec.~\ref{sec:multi-turn-grpo}. We articulate the detection logic for fatal execution cascades and rigorously derive the advantage estimation process incorporating one-sided clamping; the composite reward $r(\tau)$ and its three components $r_{\text{fmt}},\, r_{\text{acc}},\, r_{\text{query}}$ are defined directly in Sec.~\ref{sec:multi-turn-grpo}.

\begin{figure}[t]
  
  \centering
  \includegraphics[width=0.6\linewidth]{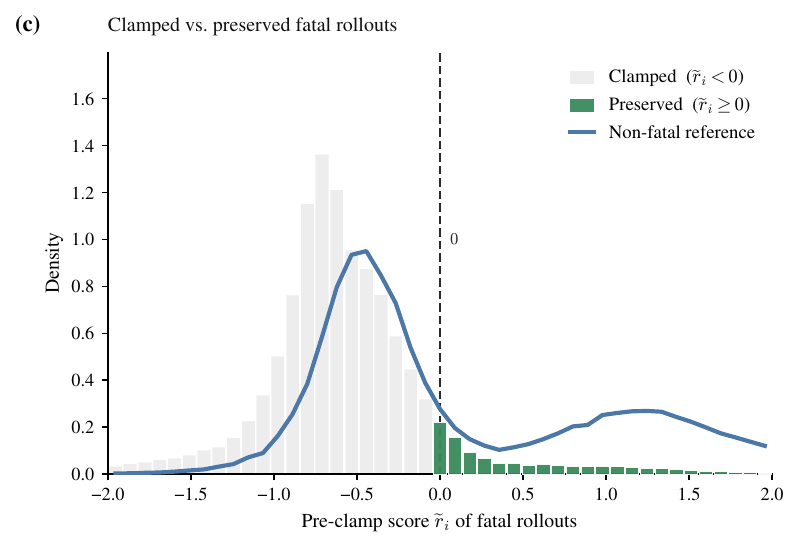}
  \caption{%
    \textbf{Aggregate distribution of clamped and preserved fatal rollouts.}
    Pre-clamp group-normalized scores $\widetilde{r}_i$ aggregated over
    $10{,}000$ groups ($G{=}16$; $47{,}978$ fatal rollouts in total).
    $91.8\%$ of fatal rollouts fall on the negative side of the clamp
    threshold (mean $\overline{\widetilde{r}}_{-}{=}-0.68$) and are zeroed
    out by $\hat{A}_i{=}\max(\widetilde{r}_i,0)$; the remaining $8.2\%$ are
    preserved (mean $\overline{\widetilde{r}}_{+}{=}+0.57$) and overlap with
    the positive mode of the non-fatal reference density (blue line:
    pre-clamp scores of all non-fatal rollouts from the same groups).
    Preservation is therefore driven by the same group-normalized score that
    GRPO already computes, and the preserved fatal prefixes sit in the same
    score regime as stronger non-fatal behaviors -- not an ad-hoc heuristic.%
  }
  \label{fig:fatal-clamping-aggregate}
  
\end{figure}

\subsection{Fatal Step Detection Logic}
\label{apdx:fatal-detail}

Unconstrained interactions with external environments inevitably yield ``fatal'' states---irrecoverable error cascades rendering subsequent reasoning invalid. We articulate the procedure for computing the fatal step index $f_i$, which drives the fatal-aware token mask (Eq.~\ref{eq:fatal-mask}).

For a given trajectory $\tau_i$ of length $L_i$, we instantiate a stateful error counter $n_{\text{err}}$ initialized to zero. At each sequential step $l \in \{0, 1, \dots, L_i\}$, the counter transitions according to:
\begin{equation}
  n_{\text{err}}^{(l)} \;=\;
  \begin{cases}
    n_{\text{err}}^{(l-1)} + 1 & \text{if step } l \text{ triggers a tool-execution error},\\
    0 & \text{otherwise}.
  \end{cases}
\end{equation}
The fatal step index is subsequently defined as the earliest step where the error threshold is breached:
\begin{equation}
    f_i = \min\{l : n_{\text{err}}^{(l)} = K\}.
\end{equation}
In our implementation, we set $K=3$. If the threshold is never reached, the trajectory is classified as non-fatal, denoted by setting $f_i = L_i+1$. 

This formulation embodies two critical design properties. First, the explicit reset condition ($n_{\text{err}}^{(l)} = 0$) ensures that \emph{isolated} transient errors---which are ubiquitous and often recoverable in realistic web environments---do not prematurely abort the trajectory. Second, the conservative threshold ($K=3$) dictates that a trajectory is only deemed fatal following three \emph{consecutive} failures, affording the autoregressive policy a robust opportunity to self-correct prior to gradient masking. Tool-execution errors triggering this counter encompass timeouts, malformed API payloads, and argument-parsing failures; these are flagged deterministically by the execution sandbox, independent of any learned heuristic.

\subsection{Advantage Derivation with One-Sided Clamping}
\label{apdx:advantage-detail}

We rigorously detail the advantage estimation protocol underpinning the objective function in Eq.~\ref{eq:our-grpo}.

\paragraph{Step 1: Unbiased Group Statistics.}
For a sampled group of $G$ rollouts originating from the identical prompt $(I_0, q)$, the empirical group mean and standard deviation are computed over the composite rewards:
\begin{equation}
  \mu_{\mathcal{G}} \;=\; \operatorname{sg}\!\left( \frac{1}{G}\sum_{j=1}^{G} r(\tau_j) \right), \qquad
  \sigma_{\mathcal{G}} \;=\; \operatorname{sg}\!\left( \sqrt{ \frac{1}{G}\sum_{j=1}^{G} \bigl(r(\tau_j) - \mu_{\mathcal{G}}\bigr)^{2} } \right).
\end{equation}
Here, $\operatorname{sg}(\cdot)$ denotes the stop-gradient operator. Consequently, $\mu_{\mathcal{G}}$ and $\sigma_{\mathcal{G}}$ act as static normalizing constants during backpropagation. Gradients flow exclusively through the importance sampling ratios evaluated on the unmasked, policy-generated tokens.

\paragraph{Step 2: Universal Group Normalization.}
Each trajectory within the group is assigned a standardized reward:
\begin{equation}
  \widetilde{r}_i \;=\; \frac{r(\tau_i) - \mu_{\mathcal{G}}}{\sigma_{\mathcal{G}} + \delta},
\end{equation}
where $\delta > 0$ provides numerical stability. Crucially, all $G$ trajectories---expressly including those truncated by the fatal condition ($f_i \le L_i$)---are incorporated into the calculation of $\mu_{\mathcal{G}}$ and $\sigma_{\mathcal{G}}$. This deliberate inclusion ensures that fatal trajectories actively shape the group-level baseline, anchoring the relative performance ranking across the entire sampled cohort.

\paragraph{Step 3: Fatal-Aware One-Sided Clamping.}
The standardized rewards are mapped to the final advantage estimates $\hat{A}_i$ via a piecewise clamping function:
\begin{equation}
  \hat{A}_i \;=\;
  \begin{cases}
    \widetilde{r}_i & \text{if } f_i = L_i{+}1 \quad\text{(non-fatal trajectory)},\\[4pt]
    \max\!\bigl(\widetilde{r}_i,\; 0\bigr) & \text{if } f_i \le L_i \quad\text{(fatal trajectory)}.
  \end{cases}
\end{equation}

\paragraph{Step 4: Gradient Dominance over Hard-Masking.}
We now formalize the informal claim that fatal-aware clamping strictly dominates the hard-masking baseline of~\citep{huang2026vision} in gradient informativeness. Let
\begin{equation}
  g_i^{\text{ours}} \;=\; \frac{1}{\sum_{t} M_{i,t}} \sum_{t=1}^{|\tau_i|} M_{i,t}\, \nabla_\theta \log \pi_\theta\!\bigl(y_{i,t} \mid h_{s(t)}\bigr)\, \hat{A}_i
  \label{eq:per-traj-grad}
\end{equation}
denote the (un-clipped) per-trajectory contribution to $\nabla_\theta \mathcal{J}(\theta)$ under our scheme, and let $g_i^{\text{hard}}$ denote the corresponding contribution under the hard-masking baseline, which discards every fatal trajectory in full so that $g_i^{\text{hard}} = 0$ whenever $f_i \le L_i$. We compare the two on the support of fatal trajectories.

\begin{proposition}[Dominance over hard-masking]
\label{prop:dominance}
Fix any fatal trajectory $\tau_i$ ($f_i \le L_i$). Then:
\begin{enumerate}[leftmargin=*,itemsep=1pt,topsep=1pt]
  \item If $\widetilde{r}_i < 0$, then $g_i^{\text{ours}} = g_i^{\text{hard}} = 0$.
  \item If $\widetilde{r}_i \ge 0$, then $g_i^{\text{hard}} = 0$, while $g_i^{\text{ours}}$ coincides with the search-augmented GRPO gradient (Eq.~\ref{eq:grpo-search}) of $\tau_i$ restricted to its viable prefix $\{t : s(t) < f_i,\; M_{\text{gen}}(y_{i,t}) = 1\}$, evaluated at the non-negative advantage $\widetilde{r}_i$.
\end{enumerate}
Consequently, $g_i^{\text{ours}}$ is weakly informative-dominant over $g_i^{\text{hard}}$: it never propagates gradient through the post-fatal suffix, never penalises the viable prefix, and strictly extracts positive reinforcement on prefixes whose group-normalised return exceeds the baseline.
\end{proposition}

\begin{proof}
By Eq.~\ref{eq:advantage}, $\widetilde{r}_i < 0$ implies $\hat{A}_i = 0$, so the surrogate term and its gradient vanish identically in Eq.~\ref{eq:per-traj-grad}; combined with $g_i^{\text{hard}} = 0$ by definition of hard-masking, this establishes (i). For $\widetilde{r}_i \ge 0$, $\hat{A}_i = \widetilde{r}_i$, while the fatal-aware mask $M_{i,t} = M_{\text{gen}}(y_{i,t}) \cdot \mathbb{1}[s(t) < f_i]$ retains exactly the policy-generated tokens of the viable prefix and zeros out (a) environment-emitted tokens and (b) all post-fatal tokens. On this support the per-token integrand of Eq.~\ref{eq:per-traj-grad} reduces token-by-token to that of Eq.~\ref{eq:grpo-search} with the same importance ratio $\rho_{i,t}$ and advantage $\widetilde{r}_i \ge 0$, so the per-token surrogates---and hence their gradients---agree. Since $g_i^{\text{hard}} = 0$ by definition, this establishes (ii).
\end{proof}

In aggregate, the one-sided clamp safely ignores invalid credit assignment while selectively harvesting positive reinforcement from prematurely truncated yet high-quality exploratory rollouts.

\paragraph{Bias of Group Statistics under Clamping.}
The asymmetric definition of $\hat{A}_i$ deserves an honest accounting. We compute $(\mu_\mathcal{G},\sigma_\mathcal{G})$ in Step~1 over \emph{all} $G$ rollouts---including fatal ones---so the standardised score $\widetilde{r}_i$ is zero-mean within the group; Step~3 then maps any negative $\widetilde{r}_i$ on a fatal trajectory to zero. Letting $\mathcal{F} \subseteq \{1,\dots,G\}$ index the fatal subset of the group,
\begin{equation}
  \frac{1}{G}\sum_{i=1}^{G} \hat{A}_i \;=\; \underbrace{\frac{1}{G}\sum_{i=1}^{G} \widetilde{r}_i}_{=\,0} \;+\; \underbrace{\frac{1}{G}\sum_{i \in \mathcal{F}} \max\!\bigl(0,\,-\widetilde{r}_i\bigr)}_{\;=\,b_\mathcal{G}\,\ge\,0},
  \label{eq:clamp-bias}
\end{equation}
so the clamp induces a non-negative bias $b_\mathcal{G}$ relative to the zero-mean GRPO baseline, in exchange for non-zero gradient on viable prefixes. We argue this trade-off is benign in our regime for two reasons. (a)~\emph{Concentration of fatal rollouts in the lower mode.} Figure~\ref{fig:fatal-clamping-aggregate} shows that $91.8\%$ of fatal rollouts have $\widetilde{r}_i < 0$ and are clamped to zero, while the residual $8.2\%$ that survive sit in the same score regime as competitive non-fatal rollouts; $b_\mathcal{G}$ is therefore dominated by samples whose viable prefix is empirically high-quality, rather than by indiscriminate inflation of low-quality fatals. (b)~\emph{Preservation of relative ranking.} Because $b_\mathcal{G}$ shifts only the fatal subset upward and never displaces non-fatal advantages, the within-group ordering between non-fatal trajectories---which is the actual signal GRPO exploits---remains intact. We therefore interpret $b_\mathcal{G}$ as a controlled positive bias that buys back gradient on prematurely truncated yet promising prefixes; replacing the clamp with an unbiased estimator (e.g., a separate baseline computed only over non-fatal rollouts, or a doubly-robust correction for $b_\mathcal{G}$) is an interesting alternative we leave for future work.

A sanity check of the resulting clamping behaviour, including a per-group-difficulty breakdown and the aggregate score distribution over $10{,}000$ groups, is reported in Sec.~\ref{sec:experiments} (Figs.~\ref{fig:fatal-clamping-cases} and~\ref{fig:fatal-clamping-aggregate}).

\section{Implementation Details}
\label{apdx:implementation-details}

Our agentic SFT pipeline builds on \textsc{LlamaFactory}~\citep{zheng2024llamafactory}, which we extend with multi-turn, tool-interleaved data collators and a Qwen3-VL-aware vision/text packing scheme so that interleaved image observations produced by visual tools (\textsc{Crop}, \textsc{Sharpen}, \textsc{SuperResolution}, \textsc{PerspectiveCorrect}, \textsc{OCR}) can be consumed verbatim by the policy. Our RL pipeline builds jointly on \textsc{rLLM}~\citep{rllm2025} and \textsc{Vision-DeepResearch}~\citep{huang2026vision}: from the former we adopt the asynchronous agent rollout architecture (decoupled trajectory generation, replay, and policy-update workers), and from the latter we adopt the multimodal trajectory abstraction. On top of these two, we implement a Qwen3-VL chat-template renderer, an interleaved image-token re-alignment routine for variable-length visual observations, and an asynchronous multi-turn rollout engine that supports tool-call interruption and resumption. The resulting pipeline is what we use to train all OpenSearch-VL variants in this paper.

\subsection{SFT Training Configuration}
\label{apdx:sft-config}

\begin{table}[!tbp]
\centering
\small
\setlength{\tabcolsep}{6pt}
\renewcommand{\arraystretch}{1.1}
\setlength{\abovecaptionskip}{14pt}
\begin{tabular}{@{}l l p{0.42\textwidth}@{}}
\toprule
\textbf{Category} & \textbf{Hyperparameter} & \textbf{Value / Setting} \\
\midrule
\multirow{4}{*}{\textbf{Model}}
& Base Model & Qwen3-VL-8B-Instruct (32B, 30B-A3B) \\
& Image Max Pixels & $262{,}144$ ($\approx 512 \times 512$) \\
& Video Max Pixels & $16{,}384$ ($\approx 128 \times 128$) \\
& Trust Remote Code & True \\
\midrule
\multirow{5}{*}{\textbf{Method}}
& Finetuning Type & Full \\
& Vision Tower Frozen & False \\
& MM Projector Frozen & False \\
& DeepSpeed Stage & ZeRO-3 \\
& Mixed Precision & bfloat16 \\
\midrule
\multirow{5}{*}{\textbf{Dataset}}
& Total Samples & 36592 \\
& Template & \texttt{qwen3\_vl} \\
& Cutoff Length & $32{,}000$ tokens \\
& Preprocessing Workers & 16 \\
& Dataloader Workers & 4 \\
\midrule
\multirow{8}{*}{\textbf{Training}}
& Batch Size per Device & 1 \\
& Gradient Accumulation Steps & 1 \\
& Effective Batch Size & 256 ($= 1 \times 1 \times 256~\text{GPUs}$) \\
& Gradient Checkpointing & True \\
& Learning Rate & $2.0 \times 10^{-5}$ \\
& Epochs & 8 \\
& LR Scheduler & cosine \\
& Warmup Ratio & 0.1 \\
\midrule
\multirow{4}{*}{\textbf{Infrastructure}}
& Total GPUs & 256 (32 nodes $\times$ 8) \\
& Orchestration & Ray + DeepSpeed ZeRO-3 \\
& Placement Strategy & PACK \\
& Resources per Worker & 1 GPU \\
\midrule
\multirow{4}{*}{\textbf{Logging / I\,O}}
& Logging Steps & 5 \\
& Checkpoint Save Steps & 400 \\
& Plot Loss & True \\
& Report Backend & TensorBoard \\
\bottomrule
\end{tabular}
\caption{Agentic SFT training configuration and hyperparameters for \textsc{OpenSearch-VL}. All three model sizes (8B dense, 32B dense, 30B-A3B MoE) use the identical recipe; only the \emph{Base Model} row differs at launch time. Hyperparameters not exercised in our pipeline---e.g.\ held-out evaluation split, LoRA/adapter settings, or layer-freezing schedules---are omitted.}
\label{tab:sft-config}
\end{table}

Table~\ref{tab:sft-config} reports the full hyperparameter configuration used for the agentic supervised fine-tuning stage across all three \textsc{OpenSearch-VL} variants. The three model sizes (8B dense, 32B dense, 30B-A3B Mixture-of-Experts) share identical training hyperparameters---only the base checkpoint path differs---so we present them in a single consolidated table. All runs use full-parameter finetuning (including the vision tower and multi-modal projector) with DeepSpeed ZeRO-3 and Ray-based orchestration over 256 GPUs (32 nodes $\times$ 8 GPUs).

\subsection{RL Training Configuration}
\label{apdx:rl-config}

Table~\ref{tab:rl-config} reports the key hyperparameters for the multi-turn fatal-aware GRPO stage across the two \textsc{OpenSearch-VL} variants we RL-finetune: \textbf{8B dense} (\texttt{Qwen3-VL-8B-Instruct}) and \textbf{30B-A3B MoE} (\texttt{Qwen3-VL-30B-A3B-Instruct}). Both runs use the same async SGLang rollout engine and the same Megatron-parallel actor, with a shared algorithm template (RLOO leave-one-out advantage under the GRPO group-relative objective, low-variance KL controller, no critic). We list only the fields set explicitly in our launch scripts; framework defaults (e.g.\ optimizer betas, reward-model paths, checkpoint-resume flags) are omitted. Rows whose value is identical for both variants are written once; rows that differ are split into two cells.

\begin{table}[!tbp]
\centering
\small
\setlength{\tabcolsep}{4pt}
\renewcommand{\arraystretch}{1.1}
\setlength{\abovecaptionskip}{14pt}
\begin{tabular}{@{}l l c c@{}}
\toprule
\textbf{Category} & \textbf{Hyperparameter} & \textbf{8B (dense)} & \textbf{30B-A3B (MoE)} \\
\midrule
\multirow{2}{*}{\textbf{Model}}
& Base Model & \texttt{Qwen3-VL-8B-Instruct} & \texttt{Qwen3-VL-30B-A3B-Instruct} \\
& Training Dtype & \multicolumn{2}{c}{bfloat16} \\
\midrule
\multirow{5}{*}{\textbf{Data}}
& Train Batch Size (prompts) & \multicolumn{2}{c}{256} \\
& Val Batch Size & 64 & 512 \\
& Max Prompt Length & \multicolumn{2}{c}{$4{,}096$ tokens} \\
& Max Response Length & \multicolumn{2}{c}{$70{,}000$ tokens} \\
& Data Seed & \multicolumn{2}{c}{3407} \\
\midrule
\multirow{7}{*}{\textbf{Rollout}}
& Engine & \multicolumn{2}{c}{SGLang (async mode)} \\
& Rollout Tensor Parallel & \multicolumn{2}{c}{4} \\
& \# Samples per Prompt ($n$) & 8 & 16 \\
& GPU Mem.\ Utilization & 0.85 & 0.65 \\
& Train / Val Temperature & \multicolumn{2}{c}{0.7 / 0.7} \\
& Train / Val Top-$p$ & \multicolumn{2}{c}{1.0 / 0.95} \\
& Top-$k$ & \multicolumn{2}{c}{$-1$ (disabled)} \\
\midrule
\multirow{11}{*}{\textbf{Policy (Actor)}}
& Strategy & \multicolumn{2}{c}{Megatron-LM} \\
& Tensor Parallel (TP) & \multicolumn{2}{c}{4} \\
& Pipeline Parallel (PP) & \multicolumn{2}{c}{2} \\
& Context Parallel (CP) & 8 & 4 \\
& Expert Parallel / ETP & --- & 8 / 1 \\
& PPO Mini-batch Size & 64 & 128 \\
& PPO Max Token Len / GPU & $74{,}576$ & $74{,}596$ \\
& Micro-batch Size / GPU & \multicolumn{2}{c}{1} \\
& Dynamic Batch Size & \multicolumn{2}{c}{True} \\
& Param / Optim / Grad Offload & \multicolumn{2}{c}{CPU} \\
& Gradient Checkpointing & \multicolumn{2}{c}{Full recompute, uniform (1 layer)} \\
\midrule
\multirow{6}{*}{\textbf{Optim.\ / Loss}}
& Actor LR & \multicolumn{2}{c}{$1\times 10^{-6}$} \\
& PPO Clip Ratio (high) & \multicolumn{2}{c}{0.28} \\
& Entropy Coefficient & \multicolumn{2}{c}{0.0} \\
& Use KL Loss & \multicolumn{2}{c}{False} \\
& KL Loss / Controller Coef & \multicolumn{2}{c}{$1{\times}10^{-3}$ / $1{\times}10^{-3}$} \\
& Loss Aggregation & \multicolumn{2}{c}{\texttt{seq-mean-token-sum}} \\
\midrule
\multirow{3}{*}{\textbf{Algorithm}}
& Advantage Estimator & \multicolumn{2}{c}{RLOO (within GRPO objective)} \\
& KL Type & \multicolumn{2}{c}{low-variance KL} \\
& Fatal-aware Masking & \multicolumn{2}{c}{True (unknown + error)} \\
\midrule
\multirow{3}{*}{\textbf{Tool-Agent}}
& \# Parallel Tasks & \multicolumn{2}{c}{256} \\
& \# Parallel Tool Calls & \multicolumn{2}{c}{$2{,}048$} \\
& Stepwise Advantage & \multicolumn{2}{c}{False} \\
\midrule
\multirow{4}{*}{\textbf{Trainer}}
& Cluster & \multicolumn{2}{c}{8 nodes $\times$ 8 GPUs = 64 GPUs} \\
& Save / Test Freq (steps) & 50 / 10 & 1 / 10 \\
& Total Epochs & 100 & 5 \\
& Critic Warmup & \multicolumn{2}{c}{0 (critic-free)} \\
\midrule
\multirow{4}{*}{\parbox{1.6cm}{\raggedright\textbf{MoE-only\\(30B-A3B)}}}
& Router Dtype & --- & fp32 \\
& Auxiliary-Loss Coef & --- & 0.01 \\
& Router Z-loss Coef & --- & $1{\times}10^{-3}$ \\
& Permute Fusion & --- & True \\
\bottomrule
\end{tabular}
\caption{Key RL training hyperparameters for the multi-turn fatal-aware GRPO stage of \textsc{OpenSearch-VL}. The 8B dense run and the 30B-A3B MoE run share the same algorithm template and rollout/actor stack; values that differ between the two (response budget per prompt, mini-batch size, context/expert parallelism, MoE-specific routing coefficients, save cadence and total epochs) are shown side-by-side. Entries marked ``---'' are not applicable (e.g.\ no expert parallelism in the dense 8B model). Framework-default fields and cluster-network environment variables are omitted.}
\label{tab:rl-config}
\end{table}

\FloatBarrier

\section{Data Curation Details}
\label{apdx:data-curation}

This section complements Sec.~\ref{sec:vqa-construction} with the full operational specification of the multi-hop VQA construction pipeline and traces the pipeline end-to-end through a concrete seed page. All statistics below are computed on the \texttt{2025-05-01} snapshot of English Wikipedia.

\subsection{Path Sampling Hyperparameters}
\label{apdx:wiki-sampling}

We fix $\tau_\text{hub}=10{,}000$ on the in-degree measured against the snapshot. Any candidate node whose incoming count exceeds $\tau_\text{hub}$ is skipped, rejecting roughly the top $0.03\%$ of all article nodes. This cut-off excludes continent-, country-, and century-level pages (e.g.\ \texttt{United\_States}, \texttt{Queensland}, \texttt{21st\_century}, \texttt{English\_language}) for which the uniqueness invariant in Eq.~\ref{eq:fuzz-invariants} is routinely violated, while preserving the long tail of entity pages that carry substantive semantic content.

\paragraph{Path length distribution.}
Path lengths are sampled as $h\sim\mathrm{Categorical}(\{2,3,4\};\,(0.4,\,0.4,\,0.2))$. Shorter paths are favoured in order to keep the downstream rollout horizons tractable; the upper bound $h=4$ is set empirically, as walks longer than four hops rarely survive the uniqueness and non-leakage checks of Eq.~\ref{eq:fuzz-invariants} without heavy resampling.

Beyond the three rules stated in the main text, the walk additionally skips:
\begin{enumerate}[leftmargin=*,itemsep=1pt,topsep=1pt]
    \item any title containing the substring \texttt{(disambiguation)}, or beginning with \texttt{List of}, \texttt{Outline of}, \texttt{Index of}, or \texttt{Timeline of};
    \item all non-article namespaces, i.e.\ \texttt{Template:}, \texttt{Category:}, \texttt{File:}, \texttt{User:}, \texttt{Help:}, \texttt{Portal:}, \texttt{Wikipedia:};
    \item redirect pages: each outgoing link is first dereferenced to its target article, and the exclusion rules above are applied to the dereferenced target rather than the surface link.
\end{enumerate}

Seeds are drawn by stratified sampling across five coarse domains---\emph{Person}, \emph{Building/Place}, \emph{Location (non-hub)}, \emph{Organism}, and \emph{Artifact}---to balance the representation of visually groundable categories. A node is eligible as a seed iff it (i) exposes an infobox; (ii) links to at least one Wikimedia Commons image of resolution no smaller than $512\times 512$; and (iii) has in-degree in $[50,\,\tau_\text{hub}]$, so that the seed is neither a dead end nor a hub.

A walk rooted at a fixed seed is retried up to $10$ times whenever it (i) hits a hub or a filtered namespace, (ii) fails to produce any descriptor for some bridge $v_j$ that survives the LLM uniqueness evaluator, or (iii) terminates at a node $v_h$ whose infobox contains no attribute meeting the six-token length bound. Seeds that exhaust all $10$ attempts are dropped from the final pool.

\subsection{Running Example: \texttt{Australia\_Zoo}}
\label{apdx:vqa-running-example}

To make the five stages of Sec.~\ref{sec:vqa-construction} concrete, we trace the pipeline end-to-end on the seed page \texttt{Australia\_Zoo} with $h=2$.

The outgoing links of \texttt{Australia\_Zoo} are first partitioned by the exclusion rules of Appendix~\ref{apdx:wiki-sampling}. Representative entries of each partition are:
\begin{itemize}[leftmargin=*,itemsep=1pt,topsep=1pt]
    \item \emph{Retained} (entity pages, in-degree below $\tau_\text{hub}$): \texttt{Bob\_Irwin}, \texttt{Steve\_Irwin}, \texttt{Terri\_Irwin}, \texttt{Bindi\_Irwin}, \texttt{Robert\_Irwin}, \texttt{The\_Crocodile\_Hunter}, \texttt{Wildlife\_Warriors}, \texttt{Beerwah,\_Queensland}, \texttt{Rosedale,\_Queensland}, \texttt{Angkor\_Wat}, \dots
    \item \emph{Hub-rejected} (in-degree $>\tau_\text{hub}$): \texttt{Queensland}, \texttt{Brisbane}, \texttt{Australia}, \texttt{United\_States}, \texttt{Zoo}.
\end{itemize}
A random walk with $h=2$ then samples
\begin{equation*}
    \underbrace{\texttt{Australia\_Zoo}}_{v_0\,(\text{anchor})}\; \xrightarrow{\;\rho_1=\text{``managed by''}\;}\; \underbrace{\texttt{Steve\_Irwin}}_{v_1\,(\text{bridge})}\; \xrightarrow{\;\rho_2=\text{``spouse of''}\;}\; \underbrace{\texttt{Terri\_Irwin}}_{v_2\,(\text{answer})}.
\end{equation*}

From the lead paragraph of $v_2$, we extract the attribute \emph{date of Australian citizenship}, yielding the short answer $a=\text{``20 November 2009''}$ (3 tokens). GPT-4o~\citep{openai2024gpt4ocard} is then prompted with the full path and its relations to synthesize the canonical question
\begin{equation*}
\begin{aligned}
q_t \;=\;\, &\text{``On what date did \emph{Terri Irwin}, the wife of \emph{Steve Irwin}---the man who}\\
            &\text{took over management of \emph{Australia Zoo} in 1991---become an Australian citizen?''}
\end{aligned}
\end{equation*}

Rewriting proceeds from the farthest bridge $v_1=\texttt{Steve\_Irwin}$ toward $v_0$. Table~\ref{tab:fuzz-candidates} lists the descriptor candidates proposed by GPT-4o~\citep{openai2024gpt4ocard} from $v_1$'s Wikipedia context and the verdicts returned by the uniqueness evaluator. The first accepted descriptor is retained (ties broken uniformly at random).
\begin{table}[h]
\centering\small
\begin{tabular}{p{0.58\linewidth}p{0.36\linewidth}}
\toprule
\textbf{Descriptor for $v_1$} & \textbf{Uniqueness verdict}\\
\midrule
``the man who took over management of [$v_0$] in 1991'' & \checkmark~unique\\
``the wildlife documentary host killed by a stingray in 2006'' & \checkmark~unique\\
``the son of the founders of [$v_0$]'' & $\times$ -- Bob and Lyn had three children (Joy, Steve, Mandy)\\
\bottomrule
\end{tabular}
\caption{Descriptor candidates for the bridge node $v_1=\texttt{Steve\_Irwin}$ and the verdicts returned by the LLM uniqueness evaluator.}
\label{tab:fuzz-candidates}
\end{table}
The accepted descriptor produces the fuzzy form
\begin{equation*}
\begin{aligned}
q_f \;=\;\, &\text{``On what date did the wife of the man who took over management of}\\
            &\text{Australia Zoo in 1991 become an Australian citizen?''}
\end{aligned}
\end{equation*}

\begin{wrapfigure}{r}{0.45\linewidth}
  \vspace{-1.0em}
  \centering
  \includegraphics[width=\linewidth]{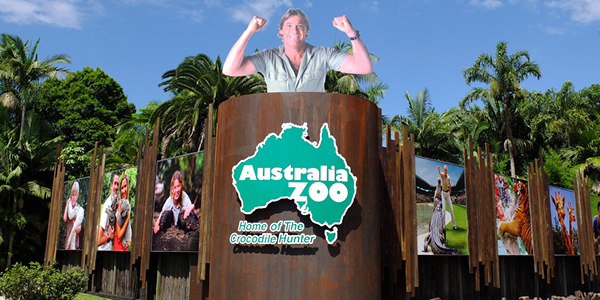}
  \caption{Representative image $I$ for the anchor $v_0=\texttt{Australia\_Zoo}$.}
  \label{apdxfig:au-zoo-anchor}
  \vspace{-0.6em} 
\end{wrapfigure}

All three invariants of Eq.~\ref{eq:fuzz-invariants} are satisfied: $a$ is preserved, the descriptor combined with ``wife of'' resolves uniquely to Terri Irwin, and $q_f$ contains no surface form or alias of any node along $P$ (e.g.\ neither \texttt{Steve Irwin}, \texttt{Terri Irwin}, nor \texttt{The Crocodile Hunter} appears in $q_f$).

We then retrieve $K=8$ candidate images from Wikimedia Commons under the query \texttt{Australia\_Zoo} and rank them by CLIP~\citep{radford2021learning} cosine similarity to the canonical description \emph{``Australia Zoo entrance''}. The top-ranked image (Fig.~\ref{apdxfig:au-zoo-anchor}, similarity $0.34$) is selected as $I$. Replacing the anchor mention in $q_f$ with a visual referring expression yields the final VQA instance
\begin{equation*}
\begin{aligned}
(I,\ q,\ a) \;=\;& \bigl(\,I,\ \text{``On what date did the wife of the man who took over management of \emph{the zoo in}}\\
                 & \text{\emph{the image} in 1991 become an Australian citizen?''},\ \text{``20 November 2009''}\,\bigr).
\end{aligned}
\end{equation*}

The instance passes all four checks of Sec.~\ref{sec:vqa-construction}: \emph{masking} ($q$ contains no entity name or alias from $\{\texttt{Australia Zoo}, \texttt{Steve Irwin}, \texttt{Terri Irwin}\}$); \emph{uniqueness} (a GPT-4o judge given only $q$ returns exactly one consistent answer); \emph{visual relevance} (CLIP similarity $0.34$ exceeds our $0.28$ threshold); and \emph{non-triviality}, verified jointly under the filtering process of Sec.~\ref{sec:filtering-enhancement}.

\subsection{Counterfactual Design Choices}
\label{apdx:vqa-counterfactuals}

We probe the necessity of the three key design choices in Sec.~\ref{sec:vqa-construction} by counterfactually relaxing each one on the running example above.

\paragraph{Anchor $=$ Answer.}
Replacing $I$ with a photograph of Terri Irwin and asking for her citizenship date reduces the task to a single reverse-image lookup: \textsc{ImageSearch} on the new $I$ returns the Wikipedia page of Terri Irwin directly, from which the citizenship date is read off in one step. The multi-hop structure collapses, independently of how the question is phrased.

\paragraph{No fuzzing.}
Retaining the canonical $q_t$ as the final question---i.e.\ skipping the fuzzy-rewriting stage---exposes all entity names in plain text, and a single \textsc{TextSearch}(``Terri Irwin Australian citizenship date'') suffices to recover $a$. The image $I$ becomes decorative rather than load-bearing.

\paragraph{No hub avoidance.}
If the walk were allowed to admit the hub \texttt{Queensland} as a bridge, a natural descriptor such as ``the zoo located in [$v_j$]'' would fail the uniqueness check---thousands of zoos satisfy this relation---forcing either unsuccessful resampling or a contrived descriptor that itself leaks the identity of \texttt{Australia\_Zoo}. Hub avoidance (Appendix~\ref{apdx:wiki-sampling}) removes this failure mode at sampling time rather than relying on the downstream filters to catch it.

\section{System Prompts}
\label{apdx:system-prompts}

This section provides the complete system prompts used in \textsc{OpenSearch-VL}. The agent system prompt (Figure~\ref{apdxfig:agent-prompt}) is shared across inference and SFT data collection, and is paired with the machine-readable tool schema (Figure~\ref{apdxfig:tool-defs}) that is injected into the model context as \texttt{<tools>...</tools>} for OpenAI-style function calling. The two reward judge prompts (Figures~\ref{apdxfig:acc-judge} and~\ref{apdxfig:query-judge}) are used during RL training to compute $r_{\text{acc}}$ and $r_{\text{query}}$, respectively. For final benchmark reporting, we additionally adopt a GPT-4o judge prompt (Figure~\ref{apdxfig:bench-judge}) that is aligned with the evaluation protocol of Vision-DeepResearch~\citep{huang2026vision}, so that our numbers remain directly comparable to prior multimodal deep-research work.

\setcounter{figure}{7}
\begin{figure}[H]
\centering
\begin{promptbox}[prompttea]{Accuracy Reward Judge Prompt ($r_{\mathrm{acc}}$)}
You are an impartial judge evaluating whether a deep research report contains the correct answer.

[Question]
\{question\}

[Correct Answer]
\{reference\_answer\}

[Deep Research Report]
\{assistant\_answer\}

Task: Determine if the deep research report contains the correct answer anywhere in its content.

Instructions:
\begin{enumerate}[leftmargin=*,nosep]
\item Read through the entire research report carefully
\item Look for the correct answer anywhere in the report (it may be embedded in paragraphs, tables, or sections)
\item Check if the information in the report is consistent with the correct answer
\item The answer does NOT need to be in a specific format or labelled as ``final answer''
\item Provide your reasoning
\item Answer with ``yes'' if the report contains the correct answer, ``no'' if it doesn't or contradicts it
\end{enumerate}

Output format:
correct: [yes/no]
reasoning: [your explanation]
\end{promptbox}
\caption{The GPT-4o judge prompt used to compute the accuracy reward $r_{\text{acc}} \in \{0,1\}$ during RL training.}
\label{apdxfig:acc-judge}
\end{figure}

\setcounter{figure}{9}
\begin{figure}[htbp]
\centering
\begin{promptbox}[promptamber]{Benchmark Evaluation Judge Prompt (GPT-4o)}
You are an impartial judge evaluating whether a deep research report contains the correct answer.

[Question]
\{question\}

[Correct Answer]
\{correct\_answer\}

[Deep Research Report]
\{response\}

Task: Determine if the deep research report contains the correct answer anywhere in its content.

Instructions:
\begin{enumerate}[leftmargin=*,nosep]
\item Read through the entire research report carefully
\item Look for the correct answer anywhere in the report (it may be embedded in paragraphs, tables, or sections)
\item Check if the information in the report is consistent with the correct answer
\item The answer does NOT need to be in a specific format or labelled as ``final answer''
\item Provide your reasoning
\item Answer with ``yes'' if the report contains the correct answer, ``no'' if it doesn't or contradicts it
\end{enumerate}

Output format:
correct: [yes/no]
reasoning: [your explanation]
\end{promptbox}
\caption{GPT-4o judge prompt used for \emph{benchmark} evaluation of \textsc{OpenSearch-VL} and all baselines. We deliberately keep this prompt aligned with the evaluation protocol released by Vision-DeepResearch~\citep{huang2026vision}, so that reported accuracies are directly comparable across systems. Although structurally similar to the RL accuracy-reward judge (Figure~\ref{apdxfig:acc-judge}), this prompt is applied \emph{post-hoc} to final agent trajectories rather than as a training signal, and uses the field names (\texttt{question}, \texttt{correct\_answer}, \texttt{response}) consumed by our evaluation script.}
\label{apdxfig:bench-judge}
\end{figure}

\begin{figure}[H]
\centering
\begin{promptbox}[promptplum]{Query-Quality Reward Judge Prompt ($r_{\mathrm{query}}$)}
You are an impartial judge evaluating the quality and utility of an agent's search trajectory.

[Original Question]
\{question\}

[Ground Truth Answer]
\{ground\_truth\}

[Agent's Final Answer]
\{prediction\}

[Search Trajectory]
\{trajectory\_summary\}

[Evaluation Criteria]
Evaluate the overall utility of the agent's search queries and retrieved results:

\begin{enumerate}[leftmargin=*,nosep]
\item \textbf{Image search utility}: Did image searches retrieve visual evidence that genuinely supports answering the question? Were the images relevant, or just noise?
\item \textbf{Text search utility}: Did text searches find relevant textual information? Were queries well-formed and targeted?
\item \textbf{Query progression}: Did the queries show logical progression---refining, narrowing, or covering different aspects? Or did they repeat / drift aimlessly?
\item \textbf{Complementarity}: Did image and text searches complement each other, providing evidence that one modality alone couldn't supply?
\item \textbf{Evidence vs.\ noise ratio}: What fraction of retrieved results actually contained useful evidence versus irrelevant content?
\end{enumerate}

Score the overall query utility from 0.0 to 1.0:
- 0.0: No useful information retrieved; all searches irrelevant or failed
- 0.3: Mostly noise with occasional marginally relevant results
- 0.5: Mixed---some useful evidence found but significant noise or inefficiency
- 0.7: Good search strategy; majority of results were relevant
- 1.0: Excellent---targeted, efficient queries that retrieved highly relevant evidence

Output format (strictly follow):
score: [a single float between 0.0 and 1.0]
reasoning: [brief explanation in 2--3 sentences]
\end{promptbox}
\caption{The GPT-4o judge prompt used to compute the query-quality reward $r_{\text{query}} \in [0,1]$ during RL training.}
\label{apdxfig:query-judge}
\end{figure}

\setcounter{figure}{6}
\begin{figure}[htbp]
\centering
\begin{promptbox}[promptblue]{Agent System Prompt (Inference / SFT Data Collection)}
You are an advanced \textbf{Visual Investigation Agent}. Answer user questions with maximum precision by proactively using image-processing and retrieval tools.

\textbf{Core Philosophy: ``Verify, Don't Guess''}
\begin{enumerate}[leftmargin=*,nosep]
\item \textbf{Tool-First Mindset}: small text $\Rightarrow$ \texttt{crop}; blurry $\Rightarrow$ \texttt{sharpen}; tilted $\Rightarrow$ \texttt{perspective\_correct}. Never rely on the internal encoder when a tool gives a sharper view.
\item \textbf{Chain Your Tools}: non-trivial queries usually require a pipeline, e.g.\ \texttt{perspective\_correct}~$\rightarrow$~\texttt{crop}~$\rightarrow$~\texttt{layout\_parsing}.
\item \textbf{External Validation}: whenever the answer depends on facts not purely visible in the pixels, you \emph{must} call \texttt{text\_search}.
\end{enumerate}

\textbf{1. Tool Calling Format}

Function signatures are provided inside \texttt{<tools>...</tools>}. Emit tool calls as one JSON object inside\\
\texttt{<tool\_call>\{"name": <f>, "arguments": <args>\}</tool\_call>}, one call per turn.

\textbf{2. Your Toolbox}

\textcolor{red}{\textbf{\{Tool List\}}}

For each of the seven tools the production prompt specifies its \emph{trigger} (when to call), \emph{params} (JSON schema), and \emph{output} (how to consume the return value). The tools fall into three families: \emph{visual perception} (\texttt{crop}, \texttt{layout\_parsing}); \emph{image enhancement} (\texttt{perspective\_correct}, \texttt{super\_resolution}, \texttt{sharpen}); and \emph{knowledge retrieval} (\texttt{text\_search}, \texttt{image\_search}; image search must be followed by text search).

\textbf{3. Thinking Protocol}

Before any action, emit a \texttt{<think>} block with: (i) \emph{analyse request}; (ii) \emph{assess image quality}---legibility, geometry, and target size, mapping any deficiency to the corresponding enhancement tool; (iii) \emph{identify information gaps}---retrieval needed?; (iv) \emph{formulate plan}---commit to a single next action.

\textbf{Critical reminders}: (a) when \texttt{layout\_parsing} returns text, treat it as ground truth---do \emph{not} fall back to visual guessing; (b) after \texttt{image\_search}, always call \texttt{text\_search} for factual detail; (c) \texttt{text\_search} summaries are already query-focused---trust them and cite the returned URLs.

\textbf{4. Workflow Recipes}
\begin{itemize}[leftmargin=*,nosep]
\item \emph{Unreadable document}: \texttt{perspective\_correct}~$\rightarrow$~\texttt{sharpen}~$\rightarrow$~\texttt{layout\_parsing}.
\item \emph{Dense chart}: \texttt{crop} (region of interest)~$\rightarrow$~\texttt{layout\_parsing}.
\item \emph{Entity identification}: \texttt{image\_search}~$\rightarrow$~\texttt{text\_search} (mandatory follow-up).
\end{itemize}

\textbf{5. Output Rules}
\begin{enumerate}[leftmargin=*,nosep]
\item \textbf{Single action per turn}; wait for its result before the next.
\item \textbf{Think first}: never emit \texttt{<tool\_call>} without a preceding \texttt{<think>}.
\item \textbf{Image refs}: initial image is \texttt{img\_1}; each tool output yields \texttt{img\_2}, \texttt{img\_3}, \ldots; always operate on the latest best version.
\item \textbf{Final answer}: emit \texttt{<response>...</response>} once evidence suffices.
\end{enumerate}

\textbf{6. Execution Example} (tool-use turn)
\begin{flushleft}\ttfamily\footnotesize
<think> Invoice img\_1 is skewed; correct perspective first. </think>\\
<tool\_call>\{"name": "perspective\_correct", "arguments": \{"image": "img\_1"\}\}</tool\_call>
\end{flushleft}
\end{promptbox}
\caption{Condensed agent system prompt used during both inference and SFT trajectory collection. The placeholder \textcolor{red}{\textbf{\{Tool List\}}} stands in for the per-tool description block of the production prompt; concrete \emph{trigger}/\emph{params}/\emph{output} content for each of the seven tools is reproduced in Figure~\ref{apdxfig:tool-defs} (machine-readable schema) and Table~\ref{tab:search-tools} (human-readable summary). Relative to the original prompt used in our codebase, we condense the three long ``critical reminder'' paragraphs into a single line, drop redundant execution examples, and merge per-tool workflow sub-rules into the core philosophy; no behavioural rule is removed.}
\label{apdxfig:agent-prompt}
\end{figure}

\begin{figure}[htbp]
\centering
\begin{toolsjson}{Tool Definitions (OpenAI-compatible JSON schema, injected as \texttt{<tools>})}
{
  "tools": [
    {
      "type": "function",
      "function": {
        "name": "crop",
        "description": "Crop a region from an image; trigger when target covers < 30% of the frame.",
        "parameters": {
          "type": "object",
          "properties": {
            "image":  {"type": "string",  "description": "Image reference, e.g. 'img_1'"},
            "x":      {"type": "integer", "description": "Top-left X coordinate"},
            "y":      {"type": "integer", "description": "Top-left Y coordinate"},
            "width":  {"type": "integer", "description": "Width of the crop region"},
            "height": {"type": "integer", "description": "Height of the crop region"}
          },
          "required": ["image", "x", "y", "width", "height"]
        }
      }
    },
    {
      "type": "function",
      "function": {
        "name": "text_search",
        "description": "Serper + JINA Reader + Qwen3-32B summarisation over top-k web pages.",
        "parameters": {
          "type": "object",
          "properties": {
            "q":     {"type": "string",  "description": "Search query keywords"},
            "hl":    {"type": "string",  "description": "Language code", "default": "en"},
            "top_k": {"type": "integer", "description": "#results to summarise", "default": 5}
          },
          "required": []
        }
      }
    },
    // ---- Remaining tools follow the same schema; full specs in the tool-suite table ----
    { "function": { "name": "layout_parsing",      "required": []         } },
    { "function": { "name": "perspective_correct", "required": ["image"]  } },
    { "function": { "name": "super_resolution",    "required": ["image"]  } },
    { "function": { "name": "sharpen",             "required": ["image"]  } },
    { "function": { "name": "image_search",        "required": ["url"]    } }
  ]
}
\end{toolsjson}
\caption{Machine-readable tool schema produced by our agent codebase and injected into the model context as an OpenAI-style \texttt{<tools>} block. Two representative tools (\texttt{crop} for visual perception and \texttt{text\_search} for knowledge retrieval) are shown in full; the remaining five follow the same schema and are collapsed to their \texttt{name}/\texttt{required} fields---their complete specifications are listed in Table~\ref{tab:search-tools}. At inference time the agent emits tool calls inside \texttt{<tool\_call>\{"name": \ldots, "arguments": \ldots\}</tool\_call>} conforming to the \texttt{parameters} schema.}
\label{apdxfig:tool-defs}
\end{figure}

\FloatBarrier

\section{Tool Definition and Usage}
\label{apdx:tool-definition}

This section details the search-oriented tools integrated within \textsc{OpenSearch-VL}. For each tool, we articulate its core functionality, formalize its input--output signature, describe its backend implementation, and contextualize its operational role within visual search trajectories. The tool set is partitioned into three functional modalities: \emph{Retrieval}, \emph{Image Enhancement}, and \emph{Attention \& Parsing}. Retrieval tools interface with the open web via remote APIs; image enhancement tools operate as lightweight, deterministic local primitives; and attention \& parsing tools combine local spatial priors with remote document-understanding services.

\subsection*{Retrieval Tools}

\begin{itemize}[leftmargin=*,itemsep=5pt]

    \item \textbf{\textsc{TextSearch}}
    \begin{itemize}
        \item \textbf{Functionality:} A composite textual retrieval mechanism comprising three pipelined stages: (i) a Serper-backend search query to retrieve top-$k$ candidate URLs; (ii) a JINA Reader invocation to fetch and normalize the HTML payloads into clean Markdown; and (iii) a Qwen3-32B summarization pass that distills a query-focused 2--4 sentence synopsis per document. The terminal output is a structured list of \texttt{[Passage $i$] (title, url, summary)} objects.
        \item \textbf{Operational Role:} Functions as the primary ``read the web'' primitive when surface-level snippets are insufficient. It is indispensable for resolving complex, multi-hop knowledge queries and retrieving long-tail factual evidence that necessitates reasoning over entire paragraphs.
    \end{itemize}

    \item \textbf{\textsc{ImageSearch}}
    \begin{itemize}
        \item \textbf{Functionality:} A visual-entity and reverse-image search tool powered by the Polaris Lens API. It accepts a publicly routable image URL and returns a structured JSON payload encompassing visually similar images, recognized entities, related domain URLs, and textual captions associated with the query image.
        \item \textbf{Operational Role:} Serves as the critical bridge transmuting purely visual queries into retrievable textual entities. When the agent determines that resolving a query hinges on identifying an unknown visual entity (e.g., a landmark, logo, or public figure), it invokes \textsc{ImageSearch} to extract external semantic grounding, which can subsequently be cross-referenced via \textsc{TextSearch}.
    \end{itemize}

\end{itemize}

\subsection*{Image Enhancement Tools}

\begin{itemize}[leftmargin=*,itemsep=5pt]

    \item \textbf{\textsc{Sharpen}}
    \begin{itemize}
        \item \textbf{Functionality:} A deterministic deblurring operator implemented via OpenCV Unsharp Masking. Parameterized by an input image and an optional sharpening intensity $\alpha$ (default $1.5$), it computes the enhanced image $I_{\text{out}} = (1{+}\alpha)\,I - \alpha\,G_\sigma * I$, where $G_\sigma$ is a Gaussian blur kernel.
        \item \textbf{Operational Role:} Deployed when the input image exhibits pervasive blur or soft edge gradients, which frequently degrade downstream optical character recognition (OCR) or object detection. As a computationally cheap, side-effect-free preprocessing step, the agent can heuristically apply \textsc{Sharpen} prior to reinvoking perceptual tools.
    \end{itemize}

    \item \textbf{\textsc{SuperResolution}}
    \begin{itemize}
        \item \textbf{Functionality:} A deep-learning-based upscaling tool employing the EDSR architecture via OpenCV's \texttt{dnn\_superres} module. It accepts an image and a discrete scale factor (default $\times 4$), emitting a high-resolution reconstruction. For robustness, if the EDSR weights are unavailable in the deployment environment, it gracefully degrades by returning the original image.
        \item \textbf{Operational Role:} Crucial for mitigating resolution bottlenecks, particularly on tightly cropped patches or inherently low-fidelity inputs (e.g., thumbnails). By synthesizing high-frequency details, it substantially elevates the reliability and extraction accuracy of subsequent \textsc{OCR} invocations.
    \end{itemize}

    \item \textbf{\textsc{PerspectiveCorrect}}
    \begin{itemize}
        \item \textbf{Functionality:} An automated perspective-rectification primitive. It executes a Canny edge detection pipeline on the grayscale projection, extracts the maximal quadrilateral contour, and computes a four-point perspective transform to warp the image into a fronto-parallel plane. If a reliable quadrilateral cannot be established, it falls back to the original image alongside a diagnostic warning.
        \item \textbf{Operational Role:} Addresses the pervasive domain shift of real-world captures, such as skewed photographs of documents, receipts, or screens. Rectifying the image geometry dramatically improves the bounding-box precision and text-recognition fidelity of downstream \textsc{OCR} parsing.
    \end{itemize}

\end{itemize}

\subsection*{Attention and Parsing Tools}

\begin{itemize}[leftmargin=*,itemsep=5pt]

    \item \textbf{\textsc{Crop}}
    \begin{itemize}
        \item \textbf{Functionality:} A deterministic spatial-attention primitive. Parameterized by an image and a bounding box coordinate tuple $(x, y, w, h)$, it extracts and isolates the specified rectangular sub-region, writing the artifact to disk for subsequent reference.
        \item \textbf{Operational Role:} Models the human cognitive mechanism of foveating onto a dense sub-region within a cluttered scene (e.g., isolating a single chart in a multi-panel figure). It is the canonical mechanism for the policy to suppress peripheral noise prior to passing a clean, localized patch to downstream tools such as \textsc{OCR}, \textsc{SuperResolution}, or \textsc{ImageSearch}.
    \end{itemize}

    \item \textbf{\textsc{OCR}}
    \begin{itemize}
        \item \textbf{Functionality:} A layout-aware optical character recognition service backed by a remote PaddleX infrastructure. It consumes a base64-encoded image alongside optional flags for chart-recognition and document-orientation classification. It emits a hierarchical list of detected text blocks (categorized into titles, body text, footnotes, etc.), yielding explicit \texttt{block\_label}, \texttt{block\_content}, and a reading-order-reconstructed \texttt{formatted\_text}.
        \item \textbf{Operational Role:} Serves as the agent's primary ``read the image'' capability. Beyond plain character recognition, its preservation of the document's logical layout is indispensable for queries dependent on structural hierarchy (e.g., distinguishing a figure caption from a section title). It is optimally invoked following upstream enhancements (\textsc{PerspectiveCorrect}, \textsc{Sharpen}, \textsc{SuperResolution}) or spatial isolation (\textsc{Crop}).
    \end{itemize}

\end{itemize}

\section{Case Study}

Figure~\ref{fig:case-study} presents a representative end-to-end trajectory of \textbf{OpenSearch-VL} on a knowledge-intensive visual question (``\emph{In what year did this bridge open?}''), whose answer is recoverable neither from the encoder's parametric knowledge nor from any single retrieval call. The agent first dispatches a \textsc{Crop} on the roadside signage to foveate on the most diagnostic sub-region, forwards the cropped patch to \textsc{ImageSearch} to identify the structure as the \textit{Kessock Bridge}, and then issues a targeted \textsc{TextSearch} that corroborates the opening year as \textbf{1982}. The trace exemplifies the compositional ``verify, don't guess'' behavior that the visual--retrieval action space $\mathcal{T}_v \cup \mathcal{T}_s$, the query-quality reward $r_{\text{query}}$, and the fatal-aware masking scheme are jointly designed to incentivize: tools are chained in an order that progressively grounds the question, and the rollout terminates as soon as cross-modal evidence converges.

\begin{figure}[htbp]
  \vspace{20pt}
  \centering
  \includegraphics[width=\linewidth]{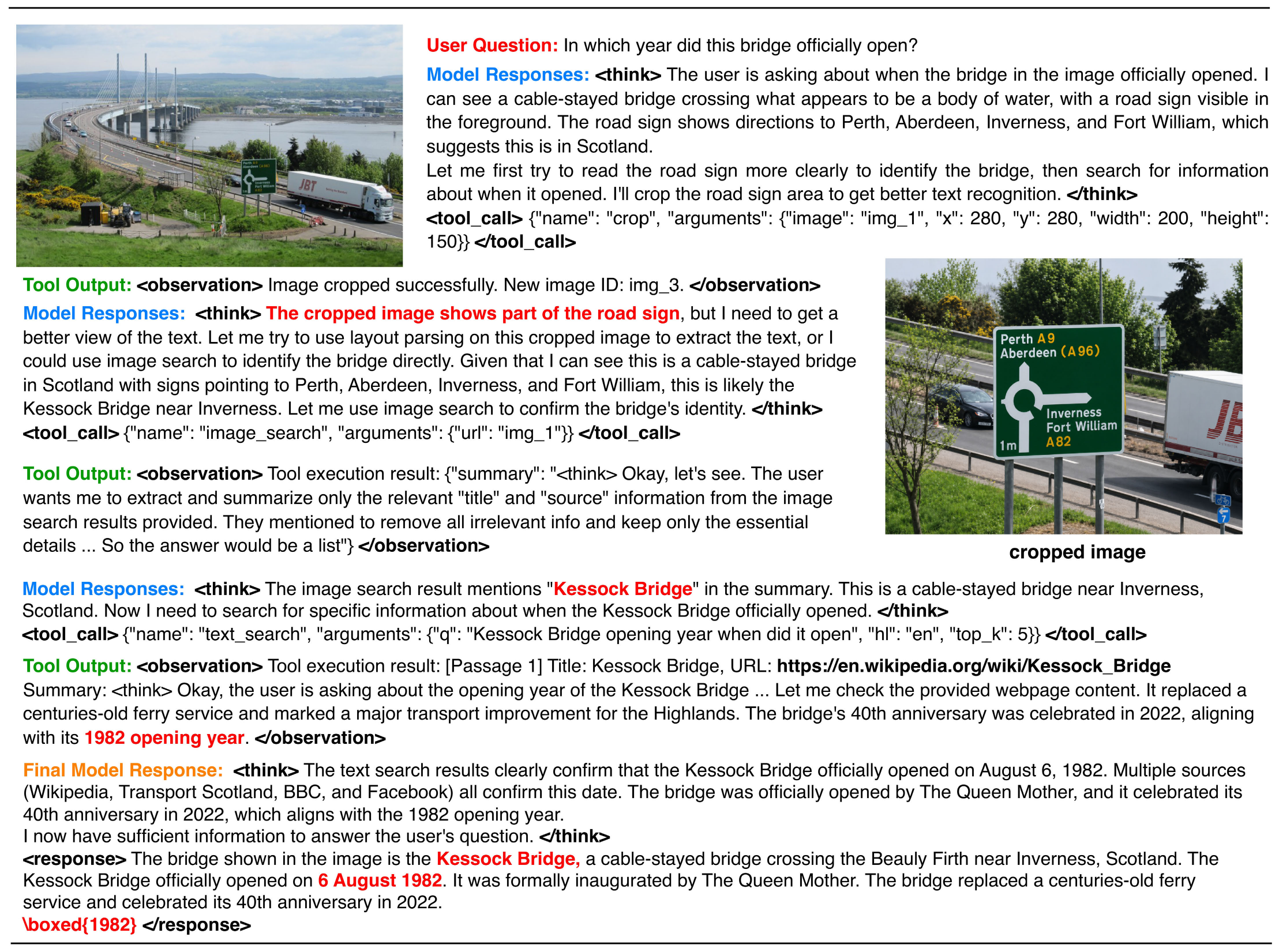}
  \caption{\textbf{Case study of \textbf{OpenSearch-VL}.}
  Given an image-based question about the opening year of a bridge, the model first inspects visual evidence and crops the road sign to obtain finer-grained location cues. It then uses image search to identify the bridge as the Kessock Bridge and issues a targeted text search to verify its official opening date. The retrieved evidence confirms that the bridge opened in 1982, leading to the final answer. This example illustrates how interleaved visual inspection, image retrieval, and textual evidence acquisition can resolve knowledge-intensive visual questions.}
  \label{fig:case-study}
\end{figure}

\stopcontents[appendix]

\clearpage
\newpage

\end{document}